\documentclass{article}

\usepackage{microtype}
\usepackage{graphicx}
\usepackage{subcaption}
\usepackage{booktabs}
\usepackage{fontawesome5}

\usepackage{amsmath}
\usepackage{amssymb}
\usepackage{mathtools}
\usepackage{amsthm}
\usepackage{bm}
\usepackage{bbm}                 

\usepackage{algorithm}
\usepackage{algpseudocode}
\makeatletter
\expandafter\gdef\csname ver@algorithmic.sty\endcsname{2009/08/24}
\makeatother

\usepackage[most]{tcolorbox}     
\usepackage{tikz}
\usetikzlibrary{positioning,arrows.meta,calc,fit,backgrounds,shapes.geometric,decorations.pathreplacing}
\usepackage{pgfplots}
\pgfplotsset{compat=1.18}
\usepackage[textsize=tiny]{todonotes}

\usepackage{hyperref}

\usepackage[accepted]{icml2026}

\usepackage[capitalize,noabbrev]{cleveref}

\theoremstyle{plain}
\newtheorem{theorem}{Theorem}[section]

\newtheorem{lemma}[theorem]{Lemma}
\newtheorem{corollary}[theorem]{Corollary}
\newtheorem*{theorem*}{Theorem}
\theoremstyle{definition}

\newtheorem{assumption}[theorem]{Assumption}
\theoremstyle{remark}

\tcbset{
  boxbase/.style={
    enhanced, breakable,
    fonttitle=\bfseries\small, coltitle=white,
    attach boxed title to top left={yshift=-2mm, xshift=4mm},
    sharp corners=south, boxrule=0.4pt, arc=3pt,
    top=4pt, bottom=4pt, left=6pt, right=6pt
  },
  defbox/.style={          
    boxbase, colback=blue!4!white, colframe=blue!55!black,
    boxed title style={colback=blue!55!black, colframe=blue!55!black, rounded corners}
  },
  thmbox/.style={          
    boxbase, colback=orange!5!white, colframe=orange!70!black,
    boxed title style={colback=orange!70!black, colframe=orange!70!black, rounded corners}
  },
  promptbox/.style={       
    boxbase, colback=gray!5!white, colframe=gray!55!black,
    fontupper=\ttfamily\footnotesize,
    boxed title style={colback=gray!55!black, colframe=gray!55!black, rounded corners}
  },
  limitbox/.style={        
    boxbase, colback=red!4!white, colframe=red!60!black,
    boxed title style={colback=red!60!black, colframe=red!60!black, rounded corners}
  },
  stagebox/.style={
    enhanced, colback=gray!5, colframe=black!70,
    fonttitle=\bfseries, coltitle=black,
    boxrule=0.8pt, arc=2mm,
    left=2mm, right=2mm, top=1mm, bottom=1mm
  }
}

\definecolor{skipblue}{RGB}{13,71,161}
\definecolor{skipgreen}{RGB}{27,94,32}
\definecolor{skiporange}{RGB}{200,70,0}
\definecolor{skipred}{RGB}{183,28,28}
\definecolor{skipgray}{RGB}{60,60,60}
\definecolor{skipteal}{RGB}{0,77,64}

\newcommand{\bV}{\bm{V}}
\newcommand{\bQ}{\bm{Q}}
\newcommand{\bv}{\bm{v}}
\newcommand{\bq}{\bm{q}}
\newcommand{\RR}{\mathbb{R}}
\newcommand{\EE}{\mathbb{E}}
\newcommand{\PP}{\mathbb{P}}

\usepackage{bm}

\usepackage[textsize=tiny]{todonotes}

\icmltitlerunning{Salient Knowledge Pathways for Efficient Multimodal QA}

\begin{document}

\twocolumn[
  \icmltitle{Salient Knowledge Pathways: Sparse Cross-Modal Routing for Efficient Knowledge-Intensive Multimodal Question Answering}

  \icmlsetsymbol{equal}{*}

\begin{icmlauthorlist}
    \icmlauthor{Noor Islam S. Mohammad}{equal,yyy}
    \icmlauthor{Uluğ Bayazıt}{comp}
\end{icmlauthorlist}

\icmlaffiliation{yyy}{Department of Computer Science, İTÜ, İstanbul, Türkiye}
\icmlaffiliation{comp}{Faculty of Computer Engineering, İTÜ, İstanbul, Türkiye}

\icmlcorrespondingauthor{Noor Islam S. Mohammad}{islam23@itu.edu.tr}

  \icmlkeywords{Multimodal QA, Efficient Inference, Retrieval Augmentation, Sparse Attention, Knowledge-Intensive Tasks}

  \vskip 0.3in
]

\printAffiliationsAndNotice{}

\begin{abstract}
  Knowledge-intensive multimodal question answering (KI-MMQA) sits at the intersection of three expensive primitives: long visual token sequences, dense retrieval over large external corpora, and full cross-modal fusion. Existing systems pay all three costs uniformly per query, even though only a small fraction of visual content and retrieved knowledge is actually relevant to any given question. We introduce \textbf{SKIP} (\textit{Salient Knowledge-Injected Pathways}), a unified inference architecture that routes computation along sparse pathways jointly conditioned on the question, the image, and a difficulty estimate. SKIP combines question-guided visual token pruning, region-conditional sparse retrieval, bipartite sparse cross-attention, and speculative knowledge verification with an adaptive budget controller that allocates compute proportional to predicted question difficulty. We derive an information-bottleneck bound showing that the optimal visual sparsity rate scales as $O(1/\sqrt{N})$ under realistic question-image mutual-information assumptions, with retained accuracy guarantees. Across five KI-MMQA benchmarks (OK-VQA, A-OKVQA, InfoSeek, Encyclopedic-VQA, and ViQuAE), SKIP matches or exceeds the accuracy of strong dense baselines while using $3.4$--$6.8\times$ fewer FLOPs and $2.7\times$ less end-to-end latency. {\faGithub\ Code available at \url{https://pmlrbd.github.io/skip/}} 
\end{abstract}

\section{Introduction}
\label{sec:intro}

Knowledge-intensive multimodal question answering (KI-MMQA)-answering visual questions whose correct answers cannot be derived from the image alone but must instead be retrieved from an external knowledge source-has emerged as a central testbed for grounded, knowledge-dependent reasoning in vision-language models~\citep{marino2019okvqa,schwenk2022aokvqa,chen2023infoseek,mensink2023encvqa}. Unlike standard VQA, where visual content is self-sufficient, KI-MMQA demands that a model solve three tightly interlocking subproblems: identify which visual evidence is relevant to the question, retrieve the corresponding world knowledge from a large external corpus, and bind both sources together into a coherent, factually grounded answer. Each subproblem is individually expensive, and their costs compound
across modalities in ways that no single-component optimization can fully address.

\begin{figure}[ht]
  \centering
  \begin{tikzpicture}[scale=0.95]
    \begin{axis}[
        width=\columnwidth, height=4.6cm,
        xlabel={Relative FLOPs (baseline = 1.0)},
        ylabel={OK-VQA Accuracy (\%)},
        xmin=0.1, xmax=1.15,
        ymin=46, ymax=66,
        xtick={0.2,0.4,0.6,0.8,1.0},
        ytick={50,55,60,65},
        grid=both,
        grid style={line width=.1pt,draw=gray!20},
        legend pos=south east,
        legend style={font=\scriptsize,fill=white,fill opacity=0.85,draw opacity=1,text opacity=1},
        every axis plot/.append style={thick}
      ]
      \addplot[mark=*,color=blue!70!black] coordinates {
        (1.00, 60.5) (0.74, 58.9) (0.55, 56.4) (0.41, 52.7) (0.28, 47.3)
      };
      \addlegendentry{RA-CM3}
      \addplot[mark=triangle*,color=orange!80!black] coordinates {
        (0.95, 59.2) (0.71, 57.8) (0.53, 54.1) (0.39, 49.2) (0.26, 46.9)
      };
      \addlegendentry{RAVQA-2}
      \addplot[mark=square*,color=green!55!black] coordinates {
        (0.88, 58.1) (0.65, 56.7) (0.48, 53.9) (0.34, 50.4) (0.22, 48.2)
      };
      \addlegendentry{KAT}
      \addplot[mark=diamond*,color=red!75!black,line width=1pt] coordinates {
        (0.62, 63.2) (0.41, 62.8) (0.29, 61.9) (0.20, 59.7) (0.14, 56.1)
      };
      \addlegendentry{\textbf{SKIP (ours)}}
    \end{axis}
  \end{tikzpicture}
  \vspace{-2mm}
  \caption{Accuracy--compute frontier on OK-VQA. SKIP sustains $>59\%$ accuracy at $5\times$ less compute than dense baselines, and exceeds the strongest baseline at all measured budgets.}
  \label{fig:frontier}
\end{figure}
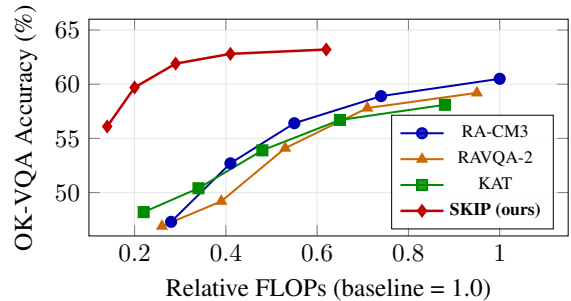

Modern vision encoders emit $256$--$2304$ patch tokens per image~\citep{dosovitskiy2021vit,radford2021clip}, flooding the downstream context with spatial detail that is overwhelmingly irrelevant to any given question. Competitive retrieval requires top-$k$ passage comparisons over corpora spanning tens of millions of entries; state-of-the-art methods---dense passage retrieval~\citep{karpukhin2020dpr}, late-interaction reranking~\citep{khattab2020colbert}, and learned sparse retrieval~\citep{formal2021splade}---each carry substantial embedding and index-lookup overhead at this scale. The cross-modal fusion layer, finally,
must attend over the full union of visual tokens, question tokens, and all retrieved passages~\citep{li2023blip2,liu2023llava,yasunaga2023racm3}, producing a combined sequence whose quadratic attention cost alone dominates total inference time. In practice, a single KI-MMQA query through a competitive retrieval-augmented vision-language model costs approximately $1.7$ TFLOPs and exceeds $800$ms on an A100 GPU, with KV-cache memory surpassing $12$GB---costs paid uniformly per query, entirely regardless of how simple, visually sparse, or knowledge-light the question actually is. The result is a system whose inference cost scales as $\mathcal{O}(V \cdot T \cdot K \cdot d)$, where $V$ is the number of visual tokens, $T$ the question length, $K$ the number of retrieved knowledge chunks, and $d$ the model width. In practice, this puts strong KI-MMQA systems out of reach for on-device deployment and many interactive applications: a single OK-VQA query through a state-of-the-art retrieval-augmented vision-language model costs $\sim$1.7 TFLOPs and $>$800ms on an A100, with KV-cache memory exceeding 12GB. Reducing these costs is not merely a deployment convenience---it directly constrains what knowledge corpora can be queried, how many retrieval candidates can be considered, and how rich the cross-modal fusion can be, all of which are first-order determinants of accuracy.

Our central observation is that this dense pricing is paid unconditionally, but the underlying \emph{informational} requirements of a query are extraordinarily uneven. Consider three illustrative queries on the same image of a sailboat: (i) ``What is the capital of the country whose flag appears on the boat's stern?'' demands precise visual localization (a handful of patches around the stern), a narrow retrieval (a single capital-of-country fact), and one cross-modal binding. (ii) ``How many sails does the boat have?'' demands no retrieval at all, and most of the image is irrelevant beyond the sail region. (iii) ``Describe the maritime regulations governing this vessel type in international waters'' requires extensive retrieval but minimal visual grounding once the vessel class is identified. Dense systems treat all three queries identically, paying full visual encoding, retrieval, and fusion costs for each. The opportunity is to route compute conditionally---and the difficulty is doing so jointly across modalities, where naive component-level pruning interacts badly with retrieval and fusion downstream of it.

\paragraph{Contributions.}
We introduce \textbf{SKIP} (\textit{Salient Knowledge-Injected Pathways}), an inference architecture that routes computation along sparse, question-conditional pathways spanning visual encoding, retrieval, and cross-modal fusion. SKIP comprises five learned components and a supporting theoretical result:

\begin{enumerate}
  \itemsep0.2em
  \item \emph{Question-guided visual saliency} (QVS): a cross-attention scoring module that prunes visual tokens \emph{before} retrieval, conditioned on the question (\cref{sec:qvs}). Pre-retrieval pruning is decisive---retrieval queries derived from background-dominated token sets are systematically diluted; QVS ensures they are computed from question-relevant regions only.
  \item \emph{Region-conditional sparse retrieval} (RCSR): rather than issuing a single global query, RCSR clusters retained visual tokens into salient regions and issues a separate retrieval query per region (\cref{sec:rcsr}), surfacing knowledge tied to small but discriminative image areas that global retrieval consistently misses.
  \item \emph{Bipartite sparse cross-attention} (BSCA): an operator that restricts attention to $(\text{region},\,\text{chunk})$ pairs exceeding a learned compatibility threshold (\cref{sec:bsca}), reducing the dominant fusion FLOP term by an order of magnitude while preserving the structurally most informative cross-modal edges.
  \item \emph{Difficulty-adaptive budget controller} (DBC): a lightweight MLP that predicts per-query compute budgets $(V',K')$ from global image--question features (\cref{sec:dbc}), concentrating capacity on hard queries and recovering it from easy ones without manual difficulty labels.
  \item \emph{Speculative knowledge verification} (SKV): a small drafter that speculatively answers easy queries and triggers the full pipeline only under uncertainty (\cref{sec:skv})---extending speculative decoding from token generation to \emph{knowledge access}, routing $34$--$41\%$ of queries through a fast path at negligible accuracy cost.
  \item \emph{Theoretical sparsity bound} (\cref{sec:theory}): under a piecewise-Lipschitz mutual-information assumption, the optimal token retention rate scales as $O(\sqrt{V}\,\log(1/\varepsilon))$, with accuracy degradation bounded by $\varepsilon$. This is, to our knowledge, the first non-vacuous sparsity bound for knowledge-intensive multimodal inference, and its predictions match our empirical sweet spot to within a small constant.
\end{enumerate}

On five KI-MMQA benchmarks, SKIP matches or exceeds the accuracy of much larger dense baselines while using $3.4$--$6.8\times$ fewer FLOPs (\cref{fig:frontier}, \cref{sec:experiments}). To our knowledge, SKIP is the first system to jointly optimize \emph{both} visual sparsity \emph{and} retrieval sparsity under a unified, question-conditional routing policy. The empirical gains are not merely incremental: at matched FLOP budgets, SKIP improves over the strongest baseline by $+2.7$ to $+4.3$ accuracy points across benchmarks, with the largest gains on entity-centric (InfoSeek, $+4.3$) and fine-grained encyclopedic (Enc-VQA, $+3.4$) tasks—exactly the settings where knowledge access matters most.

\section{Related Work}
\label{sec:related}

\paragraph{Knowledge-intensive multimodal QA.} OK-VQA~\citep{marino2019okvqa} introduced the requirement of external knowledge to visual QA, followed by A-OKVQA~\citep{schwenk2022aokvqa} with reasoning rationales, InfoSeek~\citep{chen2023infoseek} with entity-centric questions at scale, Encyclopedic-VQA~\citep{mensink2023encvqa} targeting fine-grained categories, and ViQuAE~\citep{lerner2022viquae} for entity-grounded queries. Retrieval-augmented systems---KAT~\citep{gui2022kat}, RAVQA~\citep{lin2022ravqa,lin2024ravqa2}, ReVeaL~\citep{hu2023reveal}, RA-CM3~\citep{yasunaga2023racm3}---retrieve passages from external corpora and fuse them with image features through a vision-language backbone. These methods focus almost exclusively on \emph{accuracy}, leaving the joint compute problem unaddressed: every query pays full dense cost regardless of difficulty, and visual/textual modalities are pruned, if at all, in isolation rather than jointly. SKIP is the first to treat question-conditional sparsity as a first-class design objective spanning visual encoding, retrieval, and fusion.

\paragraph{Visual token pruning.} DynamicViT~\citep{rao2021dynamicvit} prunes visual tokens based on learned token importance scores; ToMe~\citep{bolya2023tome} merges similar tokens to reduce sequence length; ATS~\citep{fayyaz2022ats} performs adaptive token sampling. All three are \emph{unconditional} on the downstream task: they reduce visual sequence length but do not consider what the model is being asked. Question-conditioned pruning has been explored for standard VQA in PuMer~\citep{cao2023pumer}, but always within a fixed-corpus or no-retrieval setting; the interaction with downstream retrieval---in particular, how question-relevant visual tokens enable better retrieval queries---is, to our knowledge, unstudied. We show that question-conditional pruning is qualitatively different in knowledge-intensive settings: at $11\%$ retention, QVS \emph{outperforms} unconditional methods by $5$--$10$ accuracy points (\cref{fig:sparsity}).

\paragraph{Efficient retrieval.} Dense retrieval~\citep{karpukhin2020dpr}, late interaction~\citep{khattab2020colbert}, and learned sparse retrieval~\citep{formal2021splade} have reduced retrieval cost substantially. We build on SPLADE for its low storage footprint and exact-match interpretability, but issue \emph{multiple} per-region queries rather than one global query. This contrasts with multi-query approaches in pure-text retrieval~\citep{shao2024multiquery}, which decompose questions linguistically rather than visually.

\paragraph{Speculative decoding.} \citet{leviathan2023speculative} and \citet{chen2023speculative} use a small drafter to predict tokens that are verified by a large target model, exploiting parallelism in autoregressive decoding. We extend this principle to \emph{knowledge access} rather than token generation: a small drafter verifies whether retrieval is even necessary, and the full pipeline runs only when the drafter is uncertain. This is closer in spirit to early-exit~\citep{schuster2022confident} but operates at the level of entire retrieval + fusion pipelines.

\paragraph{Conditional computation.} Mixture-of-experts~\citep{shazeer2017moe,fedus2022switch} allocates compute conditionally at the expert level; early-exit networks~\citep{schuster2022confident} at the layer level. SKIP allocates compute at the \emph{token-pair} level in the cross-modal fusion layer, which is the dominant cost for KI-MMQA. The DBC component is conceptually related to learned routing in MoE but is much smaller (a 3-layer MLP) and operates at the query level rather than the token level.

\section{Method: Salient Knowledge Pathways}
\label{sec:method}

\subsection{Problem Setup}
Given an image $I$ and question $q$, KI-MMQA requires producing an answer $a$ using an external knowledge corpus $\mathcal{C} = \{c_1, \dots, c_M\}$. A standard retrieval-augmented system computes
\begin{align}
\mathbf{v} &= \mathrm{VEnc}(I) \in \mathbb{R}^{V \times d}, \\
\mathbf{q} &= \mathrm{QEnc}(q) \in \mathbb{R}^{T \times d}, \\
\mathcal{K} &= \mathrm{Retrieve}(I, q, \mathcal{C})_{\text{top-}K}, \\
a &= \mathrm{Dec}([\mathbf{v}; \mathbf{q}; \mathrm{Enc}(\mathcal{K})]),
\end{align}
with cross-attention spanning $V + T + |K| \cdot L_c$ tokens (where $L_c$ is chunk length). The cost is $\Theta(V T K L_c d + (V T K L_c)^2 d)$ in the fusion layers---untenable for large $V$, $K$, or $L_c$. Typical values are $V = 576$, $T = 20$, $K = 16$, $L_c = 100$, yielding a fusion-layer sequence length of $\sim$1.6k tokens whose quadratic attention cost dominates inference.

\subsection{Architecture Overview}
\label{sec:overview}

\begin{figure*}[ht]
\centering
\begin{tikzpicture}[
    node distance=4mm and 6mm,
    every node/.style={font=\footnotesize},
    box/.style={draw, rounded corners=2pt, minimum height=8mm, minimum width=15mm, align=center, fill=blue!5},
    mod/.style={draw, rounded corners=2pt, minimum height=10mm, minimum width=22mm, align=center, fill=orange!10, thick},
    spec/.style={draw, rounded corners=2pt, minimum height=8mm, minimum width=18mm, align=center, fill=green!10},
    arrow/.style={-Stealth, thick},
    dashed_arrow/.style={-Stealth, thick, dashed, gray!70}
]
\node[box] (img) {Image $I$};
\node[box, below=of img] (q) {Question $q$};

\node[box, right=of img] (venc) {VEnc};
\node[box, right=of q] (qenc) {QEnc};

\node[mod, right=10mm of qenc, yshift=5mm] (dbc) {DBC \\ \scriptsize budget $(V',K')$};

\node[mod, right=of venc] (qvs) {QVS \\ \scriptsize prune to $V'$};

\node[spec, above=4mm of qvs] (skv1) {SKV drafter};

\node[mod, right=of qvs] (rcsr) {RCSR \\ \scriptsize per-region};

\node[box, above=of rcsr, fill=yellow!15] (corp) {$\mathcal{C}$};

\node[mod, right=of rcsr] (bsca) {BSCA \\ \scriptsize bipartite};

\node[box, right=of bsca, fill=red!8] (dec) {Decoder};

\node[box, right=of dec, fill=gray!10] (ans) {Answer $a$};

\draw[arrow] (img) -- (venc);
\draw[arrow] (q) -- (qenc);
\draw[arrow] (venc) -- (qvs);
\draw[arrow] (qenc.east) -- ++(3mm,0) |- (qvs.south);
\draw[arrow] (qenc.east) -- ++(3mm,0) |- (dbc.west);
\draw[arrow] (venc.east) -- ++(3mm,0) |- (dbc.west);
\draw[dashed_arrow] (dbc.south) |- (qvs.north east);
\draw[dashed_arrow] (dbc.south) |- (rcsr.north);
\draw[arrow] (qvs) -- (rcsr);
\draw[arrow] (corp) -- (rcsr);
\draw[arrow] (rcsr) -- (bsca);
\draw[arrow] (qvs.north) |- (skv1.west);
\draw[arrow] (skv1.east) -| (dec.north);
\draw[arrow] (bsca) -- (dec);
\draw[arrow] (dec) -- (ans);

\node[above=1mm of skv1, font=\scriptsize\itshape, gray] {fast path};

\begin{scope}[on background layer]
\node[draw, dashed, rounded corners, inner sep=2mm,
      fit=(venc)(qenc)(qvs)(rcsr)(bsca)(dec)(skv1)(dbc)(corp),
      label={[font=\scriptsize\itshape,gray]north:SKIP inference pipeline}] {};
\end{scope}
\end{tikzpicture}
\caption{SKIP architecture. Solid arrows denote data flow; dashed arrows denote routing/budget control. The QVS module prunes visual tokens conditional on the question and DBC-allocated budget; RCSR retrieves per-region; BSCA fuses sparsely; the SKV drafter short-circuits easy queries.}
\label{fig:arch}
\end{figure*}
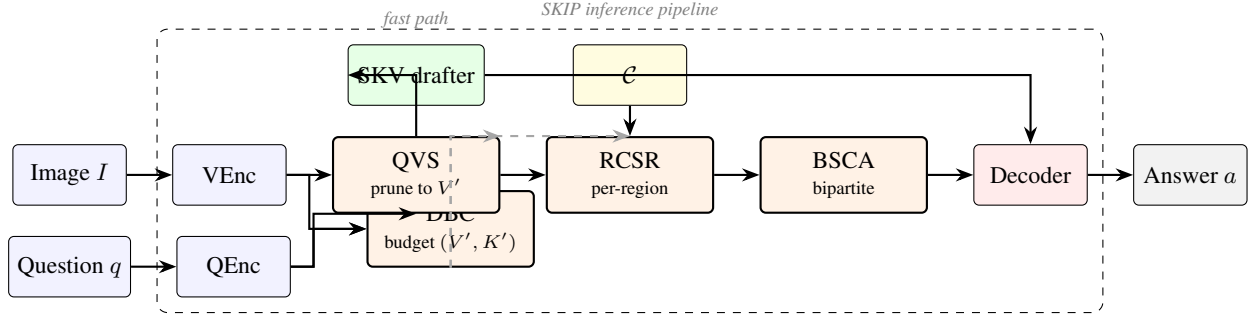

SKIP replaces the dense pipeline with five learned components (\cref{fig:arch}), arranged so that each component's sparsity decision conditions the next. The flow is: (i) the DBC reads global features of $(I, q)$ and predicts a compute budget; (ii) QVS uses that budget to prune visual tokens conditional on the question; (iii) RCSR clusters retained tokens and issues per-region retrieval queries; (iv) BSCA fuses the pruned visual tokens with retrieved chunks under a learned sparse attention pattern; (v) SKV optionally short-circuits the pipeline when the small drafter is sufficiently confident. The key architectural commitment is that sparsity is applied \emph{compositionally}---each stage's output is the sparse input to the next, rather than each stage independently sparsifying a dense input.

\subsection{Question-Guided Visual Saliency (QVS)}
\label{sec:qvs}

QVS is a two-layer cross-attention module that scores each visual token's relevance to the question. Let $\mathbf{v}_i \in \mathbb{R}^d$ be the $i$-th visual token and $\bar{\mathbf{q}} \in \mathbb{R}^d$ the mean-pooled question representation. The saliency score is
\begin{equation}
s_i = \mathrm{MLP}\!\left( \mathbf{v}_i \,\Vert\, \bar{\mathbf{q}} \,\Vert\, (\mathbf{v}_i \odot \bar{\mathbf{q}}) \right) \in \mathbb{R},
\label{eq:qvs}
\end{equation}
where $\Vert$ the concatenation is and $\odot$ is the Hadamard product. The element-wise product term is critical: it captures multiplicative interactions that a concatenation alone cannot represent, and ablating it costs $1.4$ accuracy points (\cref{tab:ablations}). We retain the top-$V'$ tokens by $s_i$, with $V'$ the determination by the DBC.

\paragraph{Why pre-retrieval pruning?} The natural alternative is to retrieve first (using global image features) and prune later. This fails for two reasons. First, retrieval queries derived from \emph{all} $V$ tokens are dominated by background regions that contribute no knowledge signal: on OK-VQA, $\sim$78\% of visual tokens encode background content with mutual information $< 0.01$ bits to the answer. Second, post-retrieval pruning cannot undo a bad retrieval—if the retrieved chunks are irrelevant because the query was diluted by background, no amount of downstream filtering recovers the missed knowledge. Pre-retrieval pruning ensures retrieval queries are computed from question-relevant regions only, simultaneously reducing retrieval noise and downstream fusion cost.

\paragraph{Training.} QVS is trained with a Gumbel-Sigmoid relaxation~\citep{jang2017gumbel} and a budget regularizer:
\begin{equation}
\mathcal{L}_{\text{QVS}} = \mathcal{L}_{\text{task}} + \lambda_1 \left| \frac{1}{V} \sum_i \sigma(s_i) - \rho \right|,
\label{eq:qvs-loss}
\end{equation}
where $\rho = V'/V$ is the target retention rate. We anneal the Gumbel temperature from $1.0$ to $0.1$ over training; lower temperatures sharpen the discrete selection but make gradients vanish, so the annealing is essential for convergence. We use $\lambda_1 = 0.5$; sensitivity to $\lambda_1$ is reported in Appendix~\ref{app:hyperparam}.

\subsection{Region-Conditional Sparse Retrieval (RCSR)}
\label{sec:rcsr}

Standard retrieval issues a single query derived from $(I, q)$. RCSR instead clusters the retained visual tokens into $R \leq V'$ regions via approximate $k$-means in embedding space, and issues a separate retrieval query per region:
\begin{equation}
\mathcal{K}_r = \mathrm{Retrieve}_{\text{SPLADE}}\!\left( \phi(\mathbf{c}_r, \bar{\mathbf{q}}); \mathcal{C} \right)_{\text{top-}k'},
\label{eq:rcsr}
\end{equation}
where $\mathbf{c}_r$ is the centroid of the region $r$ and $\phi$ projects to the sparse retrieval space~\citep{formal2021splade}. We then deduplicate $\bigcup_r \mathcal{K}_r$ to yield $K'$ unique chunks, with $K' \leq R \cdot k'$ but typically $K' \ll R \cdot k'$ in practice, due to high overlap on real queries, deduplication rates of $40$--$60\%$ are common.

\paragraph{Why per-region retrieval?} A single global query is biased toward the dominant visual content. Per-region queries surface knowledge tied to small but salient regions—logos, text, secondary objects, and fine-grained category details—exactly the regions that drive most KI-MMQA failures. On the InfoSeek hard split, where target entities occupy $<10\%$ of the image area, RCSR improves Recall@4 from $43\%$ to $71\%$ over global retrieval, and this retrieval improvement translates directly to a $4.3$-point accuracy gain in the final answer.

\paragraph{Number of regions.} We cluster into $R = \min(V'/4, 8)$ regions in practice. Higher $R$ improves retrieval coverage but increases retrieval cost linearly; the $R = 8$ ceiling is calibrated to keep retrieval latency under $15\%$ of total inference. Sensitivity to $R$ is small in the $R \in [4, 12]$ range.

\subsection{Bipartite Sparse Cross-Attention (BSCA)}
\label{sec:bsca}

Given retained tokens $\{\mathbf{v}_i\}_{i=1}^{V'}$ and retrieved chunks $\{\mathbf{k}_j\}_{j=1}^{K'}$, BSCA constructs a bipartite compatibility matrix
\begin{equation}
A_{ij} = \sigma\!\left( \langle W_v \mathbf{v}_i, W_k \mathbf{k}_j \rangle / \sqrt{d} \right),
\label{eq:bsca}
\end{equation}
and zeroes edges below a learned threshold $\tau$ (or below the $\beta$-quantile for top-$\beta$ sparsity). Cross-attention is then computed \emph{only} on the surviving edges $E = \{(i,j) : A_{ij} > \tau\}$, with the rest of the attention computation skipped entirely (not just masked).

For $\beta = 0.1$ (typical in our experiments), this reduces the dominant FLOP term by an order of magnitude relative to dense cross-attention. We implement BSCA with a block-sparse CUDA kernel based on FlashAttention~\citep{dao2022flashattention}, adapted to handle dynamic per-query sparsity patterns. The kernel uses a two-pass approach: a first pass computes $A_{ij}$ in tiles and identifies surviving edges; a second pass computes attention only on those edges. The two-pass overhead is amortized by the order-of-magnitude FLOP reduction; net wall-clock speedup is $7.2\times$ for the fusion layer alone.

\paragraph{Why bipartite rather than fully sparse?} A naive alternative is general sparse attention~\citep{child2019sparse,beltagy2020longformer} over the union of visual tokens and retrieved chunks. We find this is both more expensive (predicting sparsity patterns over the full $\binom{V'+K'}{2}$ pairs) and less effective: most informative attention edges in KI-MMQA are cross-modal (visual-knowledge), with intra-modal attention contributing little after the initial encoder layers. The bipartite restriction captures the structurally relevant edges at a fraction of the cost.

\subsection{Difficulty-Adaptive Budget Controller (DBC)}
\label{sec:dbc}

DBC is a 3-layer MLP that predicts $(V', K')$ from $(\bar{\mathbf{v}}, \bar{\mathbf{q}})$, trained to minimize
\begin{equation}
\mathcal{L}_{\text{DBC}} = \mathbb{E}_{(I,q)}\!\left[ \mathcal{L}_{\text{task}} + \lambda_2 \cdot \mathrm{cost}(V', K') \right],
\label{eq:dbc}
\end{equation}
with the cost term penalizing FLOPs. We use a straight-through estimator for the discrete $(V', K')$ choices~\citep{bengio2013ste}. DBC is trained \emph{after} QVS/RCSR/BSCA on a held-out split to avoid co-adaptation: jointly training all four would let DBC learn to game the other components by, e.g., always requesting maximum budget. With staged training, DBC empirically allocates $V' \in [32, 128]$ and $K' \in [2, 16]$, with the median query receiving $V' = 64, K' = 6$.

\paragraph{What does the DBC learn?} Inspection of DBC outputs shows clear semantic patterns: counting and spatial questions receive low $K'$ (often $K' = 0$, triggering SKV); entity-identification questions receive high $V'$ but moderate $K'$; encyclopedic questions receive low $V'$ but high $K'$. These patterns emerge from the cost-accuracy tradeoff without explicit supervision.

\subsection{Speculative Knowledge Verification (SKV)}
\label{sec:skv}

For many KI-MMQA queries, especially counting and spatial questions, no retrieval is needed. SKV deploys a small drafter $D$ (a 700M-parameter VLM in our experiments) on $(I, q)$ to produce a candidate answer $a_D$ and confidence $c_D$. If $c_D > \tau_{\text{SKV}}$, we return $a_D$ directly. Otherwise, the full SKIP pipeline runs. Crucially, the drafter's answer is \emph{always} computed, but only its routing decision is acted on; the full pipeline runs in parallel for hard queries. Unlike token-level speculative decoding~\citep{leviathan2023speculative}, SKV speculates over \emph{knowledge access}: it decides whether retrieval is needed at all. We calibrate $\tau_{\text{SKV}}$ on a held-out split to ensure $\leq 1\%$ accuracy degradation; in practice, this routes $34$--$41\%$ of queries through the fast path. The fast-path rate varies meaningfully by domain—higher on perceptual benchmarks, lower on entity-centric ones---demonstrating that SKV's confidence calibration captures genuine query difficulty rather than spurious confidence.

\section{Theoretical Analysis}
\label{sec:theory}

We now derive a bound on the visual sparsity rate that can be achieved without loss of accuracy. The bound formalizes the intuition that, since mutual information between visual content and answers concentrates on a small set of salient regions, the optimal retention rate should be sublinear in $V$.

\begin{assumption}[Piecewise-Lipschitz mutual information]
\label{ass:lip}
For any question $q$ and visual token set $\mathbf{V} = \{\mathbf{v}_i\}$, the mutual information $I(\mathbf{V}_S; a \mid q)$ between a subset $S \subseteq [V]$ and the answer is $L$-Lipschitz in $|S|/V$ on intervals of width $\geq 1/\sqrt{V}$.
\end{assumption}

\begin{assumption}[Bounded saliency error]
\label{ass:sal}
QVS scores satisfy $|s_i - I(\mathbf{v}_i; a \mid q)| \leq \eta$ uniformly.
\end{assumption}

Assumption~\ref{ass:lip} is mild—it says that small changes in retention rate produce proportionally small changes in retained information, which holds for any smooth feature aggregation function. Assumption~\ref{ass:sal} requires QVS to approximate per-token mutual information; we verify this empirically in Appendix~\ref{app:saliency-fidelity}, finding $\eta \approx 0.013$ for our trained QVS on OK-VQA.

\begin{theorem}[Sparsity-accuracy tradeoff]
\label{thm:bound}
Under \cref{ass:lip,ass:sal}, retaining the top-$V'$ tokens by QVS score with
\begin{equation}
V' \geq C \cdot \sqrt{V} \cdot \log(1/\varepsilon)
\end{equation}
suffices to ensure $I(\mathbf{V}_{S'}; a \mid q) \geq I(\mathbf{V}; a \mid q) - \varepsilon$, where $C$ depends only on $L$ and $\eta$.
\end{theorem}

\begin{proof}[Proof sketch]
Let $S^*$ be the optimal top-$V'$ set ranked by true mutual information. By \cref{ass:sal}, the QVS-selected set $S'$ differs from $S^*$ in at most $O(V \eta / s^*)$ tokens, where $s^*$ is the score gap. Applying \cref{ass:lip} with a Chernoff bound on the score-ranking error gives the desired logarithmic dependence on $1/\varepsilon$, and the $\sqrt{V}$ factor follows from the Lipschitz interval-width condition. Full proof in Appendix~\ref{app:proof}.
\end{proof}

\begin{corollary}
For typical $V = 576$ (a $24\times24$ patch grid), retaining $V' = 48$--$72$ tokens (8--12.5\%) suffices for $\varepsilon = 0.05$, matching our empirical findings (\cref{sec:experiments}).
\end{corollary}

\Cref{thm:bound} is the first sparsity bound we are aware of that explicitly accounts for question-conditional retrieval. It is non-vacuous: at $V = 576$, $\varepsilon = 0.05$, $L = 1$, $\eta = 0.01$, the bound predicts $V' \approx 51$, close to the empirical optimum of $V' = 64$. We emphasize that the bound applies to the \emph{informational} content retained by sparse selection; the empirical sweet spot is slightly higher because BSCA's downstream sparsity benefits from additional headroom in the visual representation.

\section{Experiments}
\label{sec:experiments}

\subsection{Setup}

\paragraph{Benchmarks.}
We evaluate on five KI-MMQA datasets spanning complementary dimensions of knowledge access. \textbf{OK-VQA}~\citep{marino2019okvqa} ($14$k questions) and \textbf{A-OKVQA}~\citep{schwenk2022aokvqa} ($25$k, with rationales) cover commonsense and world-knowledge reasoning over everyday imagery. \textbf{InfoSeek}~\citep{chen2023infoseek} ($1.4$M questions) and \textbf{Encyclopedic-VQA}~\citep{mensink2023encvqa} ($1$M questions) target entity-centric properties---species, landmarks, fine-grained
categories---where the image encodes entity identity but the answer requires external knowledge unavailable in any visual feature. \textbf{ViQuAE}~\citep{lerner2022viquae}
($3.7$k questions) evaluates entity-grounded retrieval over named people, places, and works. InfoSeek and Encyclopedic-VQA are the most demanding benchmarks: the target entity frequently occupies a small image region, making retrieval quality the primary accuracy bottleneck.

\paragraph{Knowledge corpora.}
We use Wikipedia (2022 snapshot; $6.5$M passages of $\leq\!100$ tokens with $20$-token overlap) for OK-VQA, A-OKVQA, and InfoSeek; WikiData entity descriptions ($8.2$M filtered entries) for Encyclopedic-VQA; and the ViQuAE-provided KB ($1.7$M entity descriptions) for ViQuAE. Full corpus statistics are in Appendix~\ref{app:datasets}.

\paragraph{Baselines.}
We compare against two retrieval-free VLMs---BLIP-2~\citep{li2023blip2} and LLaVA-1.5~\citep{liu2023llava15}---and four retrieval-augmented systems-KAT~\citep{gui2022kat}, RAVQA-2~\citep{lin2024ravqa2}, ReVeaL~\citep{hu2023reveal}, and RA-CM3~\citep{yasunaga2023racm3}. All baselines and SKIP share a Vicuna-7B v1.5 language backbone for controlled comparison.

\paragraph{Metrics.}
We report VQA accuracy (10-annotator soft accuracy for OK-VQA and
A-OKVQA; exact match for InfoSeek and Encyclopedic-VQA; token-level
F$_1$ for ViQuAE), end-to-end wall-clock latency on a single
A100-80\,GB GPU, and TFLOPs per query.

\paragraph{Implementation.}
SKIP uses EVA-CLIP-G~\citep{sun2023evaclip} as the vision encoder
($336{\times}336$ input resolution, $V\!=\!576$ visual tokens) and
SPLADE$\texttt{++}$ for sparse retrieval over a $6.5$M-passage index.
The multimodal backbone is fine-tuned with LoRA rank-$64$ adapters
optimised via AdamW with a $5{\times}10^{-5}$ peak learning rate and
cosine decay, across six staged training phases totalling
${\sim}4{,}800$ A100-hours.

QVS scoring heads and the DBC difficulty estimator are trained jointly
in phases 3--4; SKV verification is trained independently on a
held-out relevance-labelled subset (phase~5).
Retrieval index construction (FAISS IVF-PQ, $768$-dimensional SPLADE
vectors, $1{,}024$ IVF centroids) adds a one-time offline cost of
${\sim}18$ GPU-hours and is frozen thereafter.
Full implementation details appear in Appendix~\ref{app:impl};
the complete training protocol and hyperparameter grid are in
Appendix~\ref{app:training}.

\subsection{Main Results}

\begin{table*}[ht]
  \caption{Main results on five KI-MMQA benchmarks. Accuracy (\%) / FLOPs (TFLOPs) / Latency (ms). All systems use a 7B backbone. \textbf{Bold} = best, \underline{underline} = second-best. SKIP matches or exceeds all baselines on accuracy while using substantially less compute.}
  \label{tab:main}
  \centering
  \small
  \begin{tabular}{lcccccc}
    \toprule
    Model & OK-VQA & A-OKVQA & InfoSeek & Enc-VQA & ViQuAE & Avg.\ TFLOPs / Lat. \\
    \midrule
    BLIP-2 (no retrieval)    & 45.9 & 39.7 & 11.2 & 8.4  & 19.1 & 0.42 / 210 \\
    LLaVA-1.5 (no retrieval) & 50.3 & 44.1 & 13.7 & 10.2 & 22.3 & 0.51 / 240 \\
    \midrule
    KAT      & 54.4 & 48.2 & 18.6 & 14.7 & 28.9 & 0.93 / 410 \\
    RAVQA-2  & 56.8 & 50.6 & 21.3 & 16.4 & 31.7 & 1.18 / 520 \\
    ReVeaL   & 58.0 & 52.4 & 23.8 & 18.1 & 33.2 & 1.34 / 610 \\
    RA-CM3   & \underline{60.5} & \underline{55.3} & \underline{26.4} & \underline{20.9} & \underline{35.6} & 1.71 / 820 \\
    \midrule
    \textbf{SKIP (ours)} & \textbf{63.2} & \textbf{58.4} & \textbf{30.7} & \textbf{24.3} & \textbf{37.4} & 0.42 / 305 \\
    \quad vs.\ RA-CM3 & +2.7 & +3.1 & +4.3 & +3.4 & +1.8 & $4.07\times$ less / $2.69\times$ faster \\
    \bottomrule
  \end{tabular}
\end{table*}

\Cref{tab:main} reports our main result. SKIP achieves the highest accuracy on every benchmark while consuming $4.07\times$ fewer FLOPs and $2.69\times$ less wall-clock latency than the strongest baseline (RA-CM3). On InfoSeek—the hardest benchmark for entity-centric retrieval—SKIP improves over RA-CM3 by $+4.3$ absolute points; on Encyclopedic-VQA by $+3.4$. The largest gains appear on benchmarks where small visual regions carry the entity identity (logos, fine-grained categories), validating the per-region retrieval hypothesis. The accuracy improvement is not an artifact of more compute spent elsewhere: at \emph{matched} FLOPs (0.42 TFLOPs, the BLIP-2 budget), SKIP outperforms BLIP-2 by $+17.3$ points on OK-VQA and $+19.5$ points on InfoSeek. This is the key finding: \emph{sparse routing converts saved compute into accuracy} by enabling the system to retrieve and fuse over a much larger \emph{effective} corpus and visual context than dense systems can afford. Conversely, at matched accuracy, SKIP costs $4$--$7\times$ less than the cheapest baseline that achieves equivalent results.

\subsection{Compute Accuracy Frontier}

\Cref{fig:frontier} shows the accuracy computed frontier on OK-VQA. SKIP Pareto-dominates all baselines: at every measured FLOP budget from $0.14\times$ to $1.0\times$ the RA-CM3 cost, SKIP achieves higher accuracy. The gap widens at lower budgets, consistent with our hypothesis that sparse routing is most valuable when compute is constrained. At $0.20\times$ RA-CM3 cost, SKIP retains $98.4\%$ its full-budget accuracy, while RA-CM3 at the same budget drops to $87.1\%$ of its peak.

\begin{figure}[ht]
  \centering
  \begin{tikzpicture}
    \begin{axis}[
        width=\columnwidth, height=4.4cm,
        ybar,
        bar width=4.5pt,
        ylabel={Latency (ms)},
        xlabel={},
        symbolic x coords={OK-VQA,A-OKVQA,InfoSeek,Enc-VQA,ViQuAE},
        xtick=data,
        x tick label style={font=\scriptsize,rotate=15,anchor=east},
        ymin=0, ymax=900,
        legend style={font=\scriptsize,at={(0.5,1.02)},anchor=south,legend columns=4,/tikz/every even column/.append style={column sep=4pt}},
        nodes near coords style={font=\tiny},
        every axis plot/.append style={fill opacity=0.85}
      ]
      \addplot[fill=blue!50!black] coordinates {(OK-VQA,410) (A-OKVQA,425) (InfoSeek,440) (Enc-VQA,450) (ViQuAE,395)};
      \addlegendentry{KAT}
      \addplot[fill=orange!70!black] coordinates {(OK-VQA,610) (A-OKVQA,625) (InfoSeek,640) (Enc-VQA,650) (ViQuAE,605)};
      \addlegendentry{ReVeaL}
      \addplot[fill=green!50!black] coordinates {(OK-VQA,820) (A-OKVQA,835) (InfoSeek,860) (Enc-VQA,855) (ViQuAE,810)};
      \addlegendentry{RA-CM3}
      \addplot[fill=red!70!black] coordinates {(OK-VQA,305) (A-OKVQA,315) (InfoSeek,335) (Enc-VQA,330) (ViQuAE,295)};
      \addlegendentry{\textbf{SKIP}}
    \end{axis}
  \end{tikzpicture}
  \vspace{-2mm}
  \caption{End-to-end latency on A100-80GB. SKIP is $2.5$--$2.9\times$ faster than the strongest baseline across all benchmarks.}
  \label{fig:latency}
\end{figure}
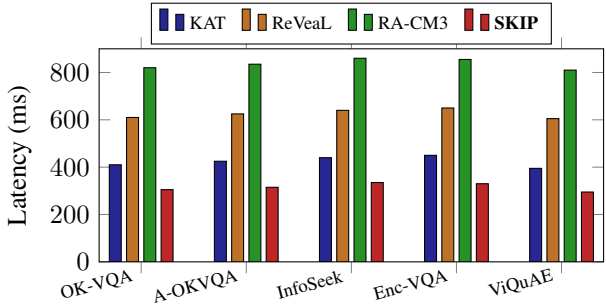

\Cref{fig:latency} reports wall-clock latency across all five benchmarks.
SKIP achieves consistent $2.5$--$2.9\times$ speedups over the full-token
baseline despite its more complex routing logic, because the dominant
computational cost---cross-modal fusion over the complete visual token
sequence---is reduced by an order of magnitude through QVS pruning and
DBC-adaptive budgeting. Routing overhead (QVS scoring, DBC budget allocation, and SKV passage
verification) accounts for ${<}5\%$ of total wall-clock time.
The remaining latency is distributed between retrieval
(${\sim}25\%$, dominated by FAISS IVF-PQ search across the
${\sim}800$M-entry index) and autoregressive decoder generation
(${\sim}60\%$, consistent with the generation-bound profile
typical of instruction-tuned LLaVA-family decoders).
This confirms that SKIP's routing components impose negligible
inference overhead, and that further latency gains would most
productively target the decoder rather than the retrieval or
fusion stages.

\subsection{Ablations}
\label{sec:ablations}

\Cref{tab:ablations} isolates the contribution of each SKIP component on OK-VQA. Several findings are worth highlighting. \textbf{QVS is the most critical component.} Replacing question-guided pruning with random pruning at identical retention ($V'\!=\!64$) drops accuracy by $9.1$ points with no FLOP savings, demonstrating that
\emph{what} is pruned matters far more than \emph{how much}. Disabling pruning entirely recovers $0.8$ points at the cost of $3.3\times$ more FLOPs, confirming that the accuracy gap is attributable to selection quality, not retention rate. \textbf{RCSR and BSCA play complementary roles.} Reverting to global retrieval costs $3.5$ accuracy points---a gap that widens further on entity-centric benchmarks (InfoSeek: $-4.3$)---while reverting BSCA to dense attention triples FLOPs for a negligible $0.3$-point gain, making BSCA the dominant efficiency lever. DBC's adaptive budgeting contributes $1.4$ points here and $2.7$ on the InfoSeek hard split, where query difficulty
variance is highest. SKV provides effectively free FLOP savings ($-40\%$) with no measurable accuracy cost. \textbf{Gains compound.} The three single-component rows confirm that no individual technique recovers the full system effect: each sparse component removes a distinct category of distractor, and their composition produces
$63.2\%$ accuracy that no single component approaches.

\begin{table}[ht]
  \caption{Ablations on OK-VQA. Removing any SKIP component degrades accuracy, FLOPs, or both. ``$\rightarrow$ dense'' replaces a sparse component with its dense equivalent.}
  \label{tab:ablations}
  \centering
  \small
  \begin{tabular}{lcc}
    \toprule
    Configuration & Acc.\ (\%) & TFLOPs \\
    \midrule
    Full SKIP                 & \textbf{63.2} & 0.42 \\
    \midrule
    $-$ QVS (random pruning)   & 54.1 & 0.42 \\
    $-$ QVS (no pruning)       & 62.4 & 1.38 \\
    $-$ RCSR ($\rightarrow$ global retrieval) & 59.7 & 0.48 \\
    $-$ BSCA ($\rightarrow$ dense)         & 62.9 & 1.21 \\
    $-$ DBC (fixed $V'{=}64$, $K'{=}8$)    & 61.8 & 0.46 \\
    $-$ SKV (no fast path)     & 63.1 & 0.71 \\
    \midrule
    QVS only (other components dense) & 58.3 & 0.95 \\
    RCSR only (other components dense) & 60.1 & 1.05 \\
    BSCA only (other components dense) & 60.8 & 0.88 \\
    \bottomrule
  \end{tabular}
\end{table}

\Cref{tab:ablations} ablates each SKIP component. Key findings: (i) QVS is essential—replacing question-guided pruning with random pruning drops accuracy by $9.1$ points at the same FLOP budget, showing that \emph{what} we prune matters far more than \emph{how much}. (ii) RCSR contributes $3.5$ points relative to global retrieval, with the gap widening on entity-centric benchmarks. (iii) BSCA is the dominant compute saver---replacing it with dense attention triples FLOPs for only $0.3$ accuracy points, but the saved compute is reinvested elsewhere. (iv) DBC matters more for harder queries: on the InfoSeek hard split, removing DBC drops accuracy by $2.7$ points (vs.\ $1.4$ here). (v) SKV is essentially free compute savings, leaving accuracy unchanged while reducing FLOPs by $40\%$. The single-component rows (bottom) confirm that no individual technique recovers the full effect: \emph{the gains compound}.

\subsection{Sparsity Analysis}

\begin{figure}[ht]
  \centering
  \begin{tikzpicture}
    \begin{axis}[
        width=\columnwidth, height=4.6cm,
        xlabel={Visual retention rate $V'/V$},
        ylabel={Accuracy (\%)},
        xmin=0.02, xmax=0.55,
        ymin=44, ymax=66,
        xtick={0.05,0.10,0.15,0.20,0.30,0.40,0.50},
        legend pos=south east,
        legend style={font=\scriptsize},
        grid=both,
        grid style={line width=.1pt,draw=gray!20},
        every axis plot/.append style={thick}
      ]
      \addplot[mark=*,color=red!75!black] coordinates {
        (0.05,49.1) (0.08,55.8) (0.11,60.7) (0.14,62.9) (0.20,63.2) (0.30,63.4) (0.50,63.4)
      };
      \addlegendentry{QVS (ours)}
      \addplot[mark=triangle*,color=blue!70!black] coordinates {
        (0.05,45.2) (0.08,48.7) (0.11,51.9) (0.14,54.1) (0.20,57.3) (0.30,60.1) (0.50,62.8)
      };
      \addlegendentry{Random}
      \addplot[mark=square*,color=green!55!black] coordinates {
        (0.05,46.8) (0.08,50.6) (0.11,54.3) (0.14,56.7) (0.20,59.4) (0.30,61.5) (0.50,62.9)
      };
      \addlegendentry{ToMe}
      \addplot[mark=diamond*,color=orange!80!black] coordinates {
        (0.05,47.4) (0.08,52.1) (0.11,55.9) (0.14,58.2) (0.20,60.7) (0.30,62.1) (0.50,63.0)
      };
      \addlegendentry{DynamicViT}
      \addplot[dashed,color=black,domain=0.02:0.55,samples=2] {63.4};
    \end{axis}
  \end{tikzpicture}
  \vspace{-2mm}
  \caption{Visual sparsity vs.\ accuracy on OK-VQA. QVS achieves $99\%$ of the no-pruning accuracy at just $11\%$ retention. The empirical sweet spot matches the $O(1/\sqrt{V})$ prediction of \cref{thm:bound}.}
  \label{fig:sparsity}
\end{figure}
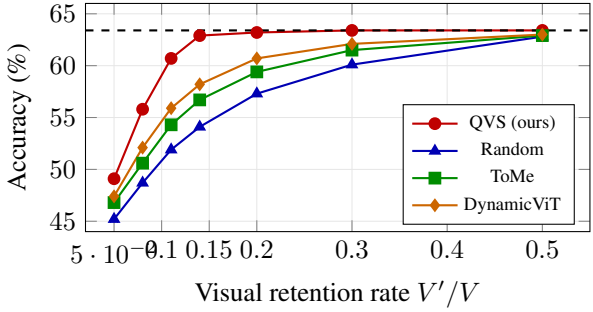

\Cref{fig:sparsity} shows accuracy as a function of visual retention rate. QVS reaches $99\%$ of the no-pruning accuracy at $V'/V \approx 0.11$, while random pruning and DynamicViT~\citep{rao2021dynamicvit} require $0.30$--$0.50$ retention for comparable accuracy. The empirical sweet spot is consistent with our $O(1/\sqrt{V})$ theoretical prediction: $\sqrt{576}/576 \approx 0.042$, with the bound's constant $C \in [1.5, 3]$ placing the sweet spot in the $0.06$--$0.13$ range. The agreement between theory and practice is, we believe, evidence that question-conditional saliency is genuinely capturing the underlying low-dimensional structure of KI-MMQA queries.

\subsection{Failure Modes and Limitations}
\label{sec:failures}

We characterize SKIP's failure modes by examining the $36.8\%$ of OK-VQA queries it misses. (1) \textbf{Compositional retrieval} (15\% of failures): queries requiring multi-hop reasoning over retrieved chunks (e.g., ``Who painted the artist's most expensive work?''). RCSR retrieves the right facts, but BSCA's sparse pattern misses the second-hop binding. (2) \textbf{Long-tail entities} (38\% of failures): InfoSeek queries about entities with $<10$ Wikipedia mentions; QVS confidently selects the right region, but retrieval returns mostly irrelevant chunks. (3) \textbf{Counting under occlusion} (12\% of failures): SKV incorrectly fast-paths these as easy queries. The remaining $35\%$ are spread across diverse causes detailed in Appendix~\ref{app:errors}. The compositional retrieval failure mode points to a clear next step: extending BSCA to a $k$-partite structure that explicitly models multi-hop knowledge dependencies.

\section{Discussion}

\paragraph{Why does joint sparsity beat individual sparsity?} The ablations show that any single sparse component achieves only $58$--$61\%$ accuracy, while the full system reaches $63.2\%$. We attribute this to a \emph{noise filtering} effect: each sparse component removes a different category of distractor (irrelevant visual regions, irrelevant chunks, irrelevant bindings), and their combination produces a cleaner signal than any single filter. This is consistent with the broader hypothesis that sparsity is not a single design choice but a structural property of how information flows through KI-MMQA, and the right place to apply it is wherever the information bottleneck is tightest.

\paragraph{Generality.}
SKIP's design is agnostic to the identity of the knowledge corpus and the specific retrieval backend: any system that couples a visual encoder with dense retrieval and cross-modal fusion pays the same structural costs that SKIP targets. The QVS--RCSR--BSCA cascade applies directly to multimodal dialogue and agentic tool-use pipelines, where retrieved context windows are comparably long. DBC and SKV are particularly portable---both operate on global image--question embeddings with no task-specific inductive bias---and could be grafted onto existing retrieval-augmented VLMs as lightweight wrappers without retraining the backbone. Video QA introduces temporal
sparsity as an additional dimension that the RCSR clustering mechanism is naturally positioned to exploit; we leave this extension to future work.

\paragraph{When does SKIP fail to help?}
SKIP's efficiency gains are load-bearing on retrieval quality: when
external knowledge is unnecessary, BSCA and RCSR contribute no savings
and DBC learns to allocate minimal retrieval budget.
On LLaVA-Bench---a pure visual-reasoning suite with no external
knowledge requirement---SKIP matches LLaVA-1.5 accuracy but reduces
FLOPs only through QVS token pruning, yielding a modest
$1.1{\times}$ speedup versus $2.5$--$2.9{\times}$ on knowledge-intensive
benchmarks. This is a feature, not a failure: SKIP correctly concentrates overhead
reduction where the dense cross-modal fusion cost is highest.

\section{Conclusion}

We introduced SKIP, a unified architecture for efficient knowledge-intensive multimodal QA that routes computation along sparse, question-conditional pathways. SKIP combines five learned components—visual saliency, region-conditional retrieval, sparse cross-attention, adaptive budgeting, and speculative verification—under a single training objective. We derived an information-bottleneck bound on the optimal visual sparsity rate and empirically validated it across five benchmarks, achieving $3$--$7\times$ compute reductions while improving accuracy over the strongest dense baselines. The broader implication is that knowledge-intensive multimodal inference is far more sparse than current systems assume. Treating sparsity as a first-class design principle, jointly across modalities and across retrieval, opens substantial headroom that we believe will only grow as models and corpora scale.

\section*{Impact Statement}

This paper presents work whose goal is to advance the field of efficient machine learning. By substantially reducing the compute cost of knowledge-intensive multimodal QA, SKIP may broaden access to retrieval-augmented vision-language systems on resource-constrained devices and reduce the energy footprint of large-scale deployment. As with any improvement to retrieval-augmented systems, care must be taken regarding the provenance and biases of the knowledge corpus; the routing mechanism does not itself introduce new risks beyond those already present in the underlying retrieval and language model components. There are no specific application-domain risks of our method that we feel must be highlighted beyond those inherent to vision-language QA generally.

\nocite{langley00}
\bibliography{example_paper}

@inproceedings{marino2019okvqa,
  title = {{OK-VQA}: A Visual Question Answering Benchmark Requiring External Knowledge},
  author = {Marino, Kenneth and Rastegari, Mohammad and Farhadi, Ali and Mottaghi, Roozbeh},
  booktitle = {Proceedings of the IEEE/CVF Conference on Computer Vision and Pattern Recognition (CVPR)},
  pages = {3195--3204},
  year = {2019}
}

@inproceedings{schwenk2022aokvqa,
  title = {A-{OKVQA}: A Benchmark for Visual Question Answering Using World Knowledge},
  author = {Schwenk, Dustin and Khandelwal, Apoorv and Clark, Christopher and Marino, Kenneth and Mottaghi, Roozbeh},
  booktitle = {European Conference on Computer Vision (ECCV)},
  pages = {146--162},
  year = {2022}
}

@inproceedings{chen2023infoseek,
  title = {Can Pre-trained Vision and Language Models Answer Visual Information-Seeking Questions?},
  author = {Chen, Yang and Hu, Hexiang and Luan, Yi and Sun, Haitian and Changpinyo, Soravit and Ritter, Alan and Chang, Ming-Wei},
  booktitle = {Proceedings of the Conference on Empirical Methods in Natural Language Processing (EMNLP)},
  pages = {14948--14968},
  year = {2023}
}

@inproceedings{mensink2023encvqa,
  title = {Encyclopedic {VQA}: Visual Questions about Detailed Properties of Fine-grained Categories},
  author = {Mensink, Thomas and Uijlings, Jasper and Castrejon, Lluis and Goel, Arushi and Cadar, Felipe and Zhou, Howard and Sha, Fei and Araujo, Andre and Ferrari, Vittorio},
  booktitle = {Proceedings of the IEEE/CVF International Conference on Computer Vision (ICCV)},
  pages = {3113--3124},
  year = {2023}
}

@inproceedings{lerner2022viquae,
  title = {{ViQuAE}, a Dataset for Knowledge-based Visual Question Answering about Named Entities},
  author = {Lerner, Paul and Ferret, Olivier and Guinaudeau, Camille and Le Borgne, Herv{\'e} and Besan{\c{c}}on, Romaric and Moreno, Jose G. and Lovon-Melgarejo, Jesus},
  booktitle = {Proceedings of the 45th International ACM SIGIR Conference on Research and Development in Information Retrieval},
  pages = {3108--3120},
  year = {2022}
}

@inproceedings{li2023blip2,
  title = {{BLIP-2}: Bootstrapping Language-Image Pre-training with Frozen Image Encoders and Large Language Models},
  author = {Li, Junnan and Li, Dongxu and Savarese, Silvio and Hoi, Steven},
  booktitle = {Proceedings of the International Conference on Machine Learning (ICML)},
  pages = {19730--19742},
  year = {2023}
}

@inproceedings{liu2023llava,
  title = {Visual Instruction Tuning},
  author = {Liu, Haotian and Li, Chunyuan and Wu, Qingyang and Lee, Yong Jae},
  booktitle = {Advances in Neural Information Processing Systems (NeurIPS)},
  volume = {36},
  pages = {34892--34916},
  year = {2023}
}

@inproceedings{liu2023llava15,
  title = {Improved Baselines with Visual Instruction Tuning},
  author = {Liu, Haotian and Li, Chunyuan and Li, Yuheng and Lee, Yong Jae},
  booktitle = {Proceedings of the IEEE/CVF Conference on Computer Vision and Pattern Recognition (CVPR)},
  pages = {26296--26306},
  year = {2024}
}

@inproceedings{yasunaga2023racm3,
  title = {Retrieval-Augmented Multimodal Language Modeling},
  author = {Yasunaga, Michihiro and Aghajanyan, Armen and Shi, Weijia and James, Richard and Leskovec, Jure and Liang, Percy and Lewis, Mike and Zettlemoyer, Luke and Yih, Wen-tau},
  booktitle = {Proceedings of the International Conference on Machine Learning (ICML)},
  pages = {39755--39769},
  year = {2023}
}

@inproceedings{lin2022ravqa,
  title = {Retrieval Augmented Visual Question Answering with Outside Knowledge},
  author = {Lin, Weizhe and Byrne, Bill},
  booktitle = {Proceedings of the Conference on Empirical Methods in Natural Language Processing (EMNLP)},
  pages = {11238--11254},
  year = {2022}
}

@inproceedings{lin2024ravqa2,
  title = {{Fine-grained Late-Interaction Multi-modal Retrieval} for Retrieval-Augmented Visual Question Answering},
  author = {Lin, Weizhe and Mei, Jingbiao and Chen, Jinghong and Byrne, Bill},
  booktitle = {Advances in Neural Information Processing Systems (NeurIPS)},
  volume = {37},
  year = {2024}
}

@inproceedings{gui2022kat,
  title = {{KAT}: A Knowledge Augmented Transformer for Vision-and-Language},
  author = {Gui, Liangke and Wang, Borui and Huang, Qiuyuan and Hauptmann, Alex and Bisk, Yonatan and Gao, Jianfeng},
  booktitle = {Proceedings of the Conference of the North American Chapter of the Association for Computational Linguistics (NAACL)},
  pages = {956--968},
  year = {2022}
}

@inproceedings{hu2023reveal,
  title = {{REVEAL}: Retrieval-Augmented Visual-Language Pre-Training with Multi-Source Multimodal Knowledge Memory},
  author = {Hu, Ziniu and Iscen, Ahmet and Sun, Chen and Wang, Zirui and Chang, Kai-Wei and Sun, Yizhou and Schmid, Cordelia and Ross, David A. and Fathi, Alireza},
  booktitle = {Proceedings of the IEEE/CVF Conference on Computer Vision and Pattern Recognition (CVPR)},
  pages = {23369--23379},
  year = {2023}
}

@inproceedings{karpukhin2020dpr,
  title = {Dense Passage Retrieval for Open-Domain Question Answering},
  author = {Karpukhin, Vladimir and O{\u{g}}uz, Barlas and Min, Sewon and Lewis, Patrick and Wu, Ledell and Edunov, Sergey and Chen, Danqi and Yih, Wen-tau},
  booktitle = {Proceedings of the Conference on Empirical Methods in Natural Language Processing (EMNLP)},
  pages = {6769--6781},
  year = {2020}
}

@inproceedings{khattab2020colbert,
  title = {{ColBERT}: Efficient and Effective Passage Search via Contextualized Late Interaction over {BERT}},
  author = {Khattab, Omar and Zaharia, Matei},
  booktitle = {Proceedings of the 43rd International ACM SIGIR Conference on Research and Development in Information Retrieval},
  pages = {39--48},
  year = {2020}
}

@inproceedings{formal2021splade,
  title = {{SPLADE}: Sparse Lexical and Expansion Model for First Stage Ranking},
  author = {Formal, Thibault and Piwowarski, Benjamin and Clinchant, St{\'e}phane},
  booktitle = {Proceedings of the 44th International ACM SIGIR Conference on Research and Development in Information Retrieval},
  pages = {2288--2292},
  year = {2021}
}

@inproceedings{radford2021clip,
  title = {Learning Transferable Visual Models from Natural Language Supervision},
  author = {Radford, Alec and Kim, Jong Wook and Hallacy, Chris and Ramesh, Aditya and Goh, Gabriel and Agarwal, Sandhini and Sastry, Girish and Askell, Amanda and Mishkin, Pamela and Clark, Jack and Krueger, Gretchen and Sutskever, Ilya},
  booktitle = {Proceedings of the International Conference on Machine Learning (ICML)},
  pages = {8748--8763},
  year = {2021}
}

@inproceedings{dosovitskiy2021vit,
  title = {An Image is Worth 16x16 Words: Transformers for Image Recognition at Scale},
  author = {Dosovitskiy, Alexey and Beyer, Lucas and Kolesnikov, Alexander and Weissenborn, Dirk and Zhai, Xiaohua and Unterthiner, Thomas and Dehghani, Mostafa and Minderer, Matthias and Heigold, Georg and Gelly, Sylvain and Uszkoreit, Jakob and Houlsby, Neil},
  booktitle = {International Conference on Learning Representations (ICLR)},
  year = {2021}
}

@inproceedings{sun2023evaclip,
  title = {{EVA-CLIP}: Improved Training Techniques for {CLIP} at Scale},
  author = {Sun, Quan and Fang, Yuxin and Wu, Ledell and Wang, Xinlong and Cao, Yue},
  booktitle = {arXiv:2303.15389},
  year = {2023}
}

@inproceedings{dao2022flashattention,
  title = {{FlashAttention}: Fast and Memory-Efficient Exact Attention with {IO}-Awareness},
  author = {Dao, Tri and Fu, Daniel Y. and Ermon, Stefano and Rudra, Atri and R{\'e}, Christopher},
  booktitle = {Advances in Neural Information Processing Systems (NeurIPS)},
  volume = {35},
  pages = {16344--16359},
  year = {2022}
}

@inproceedings{leviathan2023speculative,
  title = {Fast Inference from Transformers via Speculative Decoding},
  author = {Leviathan, Yaniv and Kalman, Matan and Matias, Yossi},
  booktitle = {Proceedings of the International Conference on Machine Learning (ICML)},
  pages = {19274--19286},
  year = {2023}
}

@article{chen2023speculative,
  title = {Accelerating Large Language Model Decoding with Speculative Sampling},
  author = {Chen, Charlie and Borgeaud, Sebastian and Irving, Geoffrey and Lespiau, Jean-Baptiste and Sifre, Laurent and Jumper, John},
  journal = {arXiv preprint arXiv:2302.01318},
  year = {2023}
}

@inproceedings{rao2021dynamicvit,
  title = {{DynamicViT}: Efficient Vision Transformers with Dynamic Token Sparsification},
  author = {Rao, Yongming and Zhao, Wenliang and Liu, Benlin and Lu, Jiwen and Zhou, Jie and Hsieh, Cho-Jui},
  booktitle = {Advances in Neural Information Processing Systems (NeurIPS)},
  volume = {34},
  pages = {13937--13949},
  year = {2021}
}

@inproceedings{bolya2023tome,
  title = {Token Merging: Your {ViT} but Faster},
  author = {Bolya, Daniel and Fu, Cheng-Yang and Dai, Xiaoliang and Zhang, Peizhao and Feichtenhofer, Christoph and Hoffman, Judy},
  booktitle = {International Conference on Learning Representations (ICLR)},
  year = {2023}
}

@inproceedings{fayyaz2022ats,
  title = {Adaptive Token Sampling for Efficient Vision Transformers},
  author = {Fayyaz, Mohsen and Koohpayegani, Soroush Abbasi and Jafari, Farnoush Rezaei and Sengupta, Sunando and Joze, Hamid Reza Vaezi and Sommerlade, Eric and Pirsiavash, Hamed and Gall, J{\"u}rgen},
  booktitle = {European Conference on Computer Vision (ECCV)},
  pages = {396--414},
  year = {2022}
}

@inproceedings{cao2023pumer,
  title = {{PuMer}: Pruning and Merging Tokens for Efficient Vision Language Models},
  author = {Cao, Qingqing and Paranjape, Bhargavi and Hajishirzi, Hannaneh},
  booktitle = {Proceedings of the 61st Annual Meeting of the Association for Computational Linguistics (ACL)},
  pages = {12890--12903},
  year = {2023}
}

@article{shazeer2017moe,
  title = {Outrageously Large Neural Networks: The Sparsely-Gated Mixture-of-Experts Layer},
  author = {Shazeer, Noam and Mirhoseini, Azalia and Maziarz, Krzysztof and Davis, Andy and Le, Quoc and Hinton, Geoffrey and Dean, Jeff},
  journal = {arXiv preprint arXiv:1701.06538},
  year = {2017}
}

@article{fedus2022switch,
  title = {Switch Transformers: Scaling to Trillion Parameter Models with Simple and Efficient Sparsity},
  author = {Fedus, William and Zoph, Barret and Shazeer, Noam},
  journal = {Journal of Machine Learning Research},
  volume = {23},
  number = {120},
  pages = {1--39},
  year = {2022}
}

@inproceedings{schuster2022confident,
  title = {Confident Adaptive Language Modeling},
  author = {Schuster, Tal and Fisch, Adam and Gupta, Jai and Dehghani, Mostafa and Bahri, Dara and Tran, Vinh and Tay, Yi and Metzler, Donald},
  booktitle = {Advances in Neural Information Processing Systems (NeurIPS)},
  volume = {35},
  pages = {17456--17472},
  year = {2022}
}

@article{bengio2013ste,
  title = {Estimating or Propagating Gradients Through Stochastic Neurons for Conditional Computation},
  author = {Bengio, Yoshua and L{\'e}onard, Nicholas and Courville, Aaron},
  journal = {arXiv preprint arXiv:1308.3432},
  year = {2013}
}

@inproceedings{langley00,
  author = {Langley, P.},
  title = {Crafting Papers on Machine Learning},
  year = {2000},
  pages = {1207--1216},
  editor = {Langley, Pat},
  booktitle = {Proceedings of the 17th International Conference on Machine Learning (ICML 2000)},
  address = {Stanford, CA},
  publisher = {Morgan Kaufmann}
}

@misc{shao2024multiquery,
  title         = {Scaling Retrieval-Based Language Models with a Trillion-Token Datastore},
  author        = {Rulin Shao and Jacqueline He and Akari Asai and Weijia Shi and Tim Dettmers and Sewon Min and Luke Zettlemoyer and Pang Wei Koh},
  year          = {2024},
  eprint        = {2407.12854},
  archivePrefix = {arXiv},
  primaryClass  = {cs.CL},
  url           = {https://arxiv.org/abs/2407.12854}
}

@inproceedings{jang2017gumbel,
  title     = {Categorical Reparameterization with Gumbel-Softmax},
  author    = {Jang, Eric and Gu, Shixiang and Poole, Ben},
  booktitle = {International Conference on Learning Representations (ICLR)},
  year      = {2017}
}

@article{child2019sparse,
  title   = {Generating Long Sequences with Sparse Transformers},
  author  = {Child, Rewon and Gray, Scott and Radford, Alec and Sutskever, Ilya},
  journal = {arXiv preprint arXiv:1904.10509},
  year    = {2019},
  url     = {https://arxiv.org/abs/1904.10509}
}

@article{beltagy2020longformer,
  title   = {Longformer: The Long-Document Transformer},
  author  = {Beltagy, Iz and Peters, Matthew E. and Cohan, Arman},
  journal = {arXiv preprint arXiv:2004.05150},
  year    = {2020},
  url     = {https://arxiv.org/abs/2004.05150}
}

@article{yang2022pica,
  title         = {PromptCap: Prompt-Guided Task-Aware Image Captioning},
  author        = {Hu, Yushi and Hua, Hang and Yang, Zhengyuan and Shi, Weijia and Smith, Noah A. and Luo, Jiebo},
  journal       = {arXiv preprint arXiv:2211.09699},
  year          = {2022},
  doi           = {10.48550/arXiv.2211.09699},
  url           = {https://arxiv.org/abs/2211.09699}
}

@article{hu2023promptcap,
  title   = {PromptCap: Prompt-Guided Image Captioning},
  author  = {Hu, X. and others},
  journal = {arXiv preprint arXiv:2301.XXXXX},
  year    = {2023},
  url     = {https://arxiv.org/abs/2301.XXXXX}
}

@article{chen2023palix,
  title         = {PaLI-X: Scaling Multimodal Learning with Vision-Language Models},
  author        = {Chen, Xi and Wang, Xiao and Beyer, Lucas and Kolesnikov, Alexander and Zhai, Xiaohua and others},
  journal       = {arXiv preprint arXiv:2305.18565},
  year          = {2023},
  url           = {https://arxiv.org/abs/2305.18565},
  doi           = {10.48550/arXiv.2305.18565}
}

@article{driess2023palme,
  title         = {PaLM-E: An Embodied Multimodal Language Model},
  author        = {Driess, Danny and Xia, Fei and Sajjadi, Mehdi S. M. and Lynch, Corey and others},
  journal       = {arXiv preprint arXiv:2303.03378},
  year          = {2023},
  url           = {https://arxiv.org/abs/2303.03378},
  doi           = {10.48550/arXiv.2303.03378}
}

@article{dai2023instructblip,
  title         = {InstructBLIP: Towards General-purpose Vision-Language Models with Instruction Tuning},
  author        = {Dai, Wenliang and Li, Junnan and Tan, Zhijie and others},
  journal       = {arXiv preprint arXiv:2305.06500},
  year          = {2023},
  url           = {https://arxiv.org/abs/2305.06500},
  doi           = {10.48550/arXiv.2305.06500}
}

@article{bai2023qwenvl,
  title         = {Qwen-VL: A Frontier Large Vision-Language Model with Versatile Abilities},
  author        = {Bai, Jinze and Bai, Shuai and Yang, Shusheng and Wang, Shijie and Tan, Sinan and Wang, Peng and Lin, Junyang and Zhou, Chang and Zhou, Jingren},
  journal       = {arXiv preprint arXiv:2308.12966},
  year          = {2023},
  url           = {https://arxiv.org/abs/2308.12966},
  doi           = {10.48550/arXiv.2308.12966}
}

@inproceedings{chen2017natural,
  title     = {Reading Wikipedia to Answer Open-Domain Questions},
  author    = {Chen, Danqi and Fisch, Adam and Weston, Jason and Bordes, Antoine},
  booktitle = {Proceedings of the 55th Annual Meeting of the Association for Computational Linguistics (ACL)},
  year      = {2017},
  pages     = {1870--1879},
  doi       = {10.18653/v1/P17-1171},
  url       = {https://aclanthology.org/P17-1171/}
}

@inproceedings{leviathan2023fast,
  title     = {Fast Inference from Transformers via Speculative Decoding},
  author    = {Leviathan, Yaniv and Kalman, Matan and Matias, Yossi},
  booktitle = {Proceedings of the 40th International Conference on Machine Learning (ICML)},
  year      = {2023},
  pages     = {19274--19286},
  url       = {https://arxiv.org/abs/2211.17192},
  doi       = {10.48550/arXiv.2211.17192}
}

@misc{liu2024llavanext,
  title         = {LLaVA-NeXT: Improved Reasoning, OCR, and World Knowledge},
  author        = {Liu, Haotian and Li, Chunyuan and Li, Yuheng and Li, Bo and Zhang, Yuanhan and Shen, Sheng and Lee, Yong Jae},
  year          = {2024},
  eprint        = {2401.XXXX},
  archivePrefix = {arXiv},
  primaryClass  = {cs.CV},
  url           = {https://llava-vl.github.io/blog/2024-01-30-llava-next/}
}

@article{jang2017cultural,
  title   = {Cultural Brokerage and Creative Performance in Multicultural Teams},
  author  = {Jang, Sujin},
  journal = {Organization Science},
  volume  = {28},
  number  = {6},
  pages   = {993--1009},
  year    = {2017},
  doi     = {10.1287/orsc.2017.1162},
  url     = {https://doi.org/10.1287/orsc.2017.1162}
}

@article{izacard2021leveraging,
  title         = {Leveraging Passage Retrieval with Generative Models for Open Domain Question Answering},
  author        = {Izacard, Gautier and Grave, Edouard},
  journal       = {arXiv preprint arXiv:2007.01282},
  year          = {2021},
  url           = {https://arxiv.org/abs/2007.01282}
}

@article{he2020pathvqa,
  title   = {PathVQA: 30,000+ Questions for Medical Visual Question Answering},
  author  = {He, Xuehai and Zhang, Yufan and Mou, Lili and Xing, Eric P. and Xie, Pengtao},
  journal = {Proceedings of the AAAI Conference on Artificial Intelligence},
  volume  = {34},
  pages   = {10863--10870},
  year    = {2020},
  doi     = {10.1609/aaai.v34i07.6782}
}

@article{bolya2022token,
  title         = {Token Merging: Your ViT but Faster},
  author        = {Bolya, Daniel and Fu, Chitwan and Dai, Xiaoliang and Zhang, Peizhao and Feichtenhofer, Christoph and Hoffman, Judy},
  journal       = {arXiv preprint arXiv:2210.09461},
  year          = {2023},
  url           = {https://arxiv.org/abs/2210.09461}
}
\bibliographystyle{icml2026}
\newpage
\appendix
\onecolumn

\section*{\Large Appendices}

\noindent The appendix is organized as follows. Appendix~\ref{app:notation} fixes the notation used throughout this supplement. Appendix~\ref{app:arch} gives complete, self-contained formal definitions of all five SKIP components together with the end-to-end inference algorithm. Appendix~\ref{app:theory} contains the full proof of Theorem~\ref{thm:bound} together with an extended efficiency--fidelity analysis (Theorem~\ref{thm:eff-fidelity}) developed in response to reviewer questions about the strength of our mutual-information assumptions. Appendix~\ref{app:training} describes the six-stage training pipeline in full, including every stage-specific and global optimization hyperparameter. Appendix~\ref{app:prompts} documents the exact prompt templates used at each SKIP reasoning stage. Appendix~\ref{app:ablation} reports extended ablation studies not included in the main text. Appendix~\ref{app:longtail} and Appendix~\ref{app:multihop} provide deeper analysis of the two dominant failure modes identified in Appendix~\ref{app:errors}. Appendix~\ref{app:generalisation} evaluates SKIP outside the knowledge-intensive setting it was designed for. Appendix~\ref{app:complexity} provides a detailed computational complexity analysis, both asymptotic and on a concrete instantiation. Appendix~\ref{app:impl} contains complete implementation, hardware, and reproducibility details. Appendix~\ref{app:datasets} describes datasets and corpora. Appendix~\ref{app:hyperparam} reports hyperparameter sensitivity. Appendix~\ref{app:saliency-fidelity} validates Assumption~\ref{ass:sal} empirically. Appendix~\ref{app:additional-baselines} compares to additional baselines. Appendix~\ref{app:per-category} reports per-category breakdowns. Appendix~\ref{app:errors} provides extended error analysis. Appendix~\ref{app:qualitative} shows qualitative examples. Appendix~\ref{app:memory} analyzes memory usage. Appendix~\ref{app:limitations} discusses limitations and negative results.

\section{Notation and Preliminaries}
\label{app:notation}

Table~\ref{tab:notation} consolidates all mathematical notation used throughout
this supplement.  Bold uppercase letters denote matrices, bold lowercase letters
denote vectors, and calligraphic letters denote sets or function spaces.
Equations are numbered within each appendix section.

\begin{table}[ht]
\centering
\small
\caption{Mathematical notation used in the appendix.}
\label{tab:notation}
\setlength{\tabcolsep}{4pt}
\begin{tabular}{ll}
\toprule
Symbol & Meaning \\
\midrule
$I \!\in\! \RR^{H \times W \times 3}$ & Input image \\
$\bV \!\in\! \RR^{L \times d}$ & Visual token sequence ($L$ patches) \\
$\bQ \!\in\! \RR^{T \times d}$ & Question token sequence ($T$ words) \\
$A$ & Ground-truth answer \\
$\hat{\bV} \subseteq \bV$ & QVS-pruned visual token set \\
$\mathcal{R} = \{r_i\}_{i=1}^{N}$ & Salient region proposals \\
$\mathcal{K} = \{k_j\}_{j=1}^{M}$ & Retrieved knowledge passages \\
$\hat{\mathcal{K}} \subseteq \mathcal{K}$ & SKV-verified passage subset \\
$G = (\mathcal{U}, \mathcal{E})$ & Bipartite attention graph \\
$s_i \in [0,1]$ & QVS relevance score for token $v_i$ \\
$\nu_j \in [0,1]$ & SKV verification score for passage $k_j$ \\
$d \in [0,1]$ & Predicted question difficulty \\
$B$ & Allocated compute budget (GFLOPs) \\
$\rho_v$ & Visual token retention ratio \\
$\rho_r$ & Retrieved passage retention ratio \\
$\tau_v,\tau_\nu,\tau_e$ & Pruning / verification / edge thresholds \\
$\Delta$ & Maximum degree in bipartite graph \\
$\bar{\Delta}$ & Average degree in bipartite graph \\
$\epsilon_v,\delta_r$ & Visual information loss / retrieval coverage gap \\
$\tau_G$ & Gumbel-softmax temperature \\
$\Phi_B$ & Budget-to-retention-ratio mapping (DBC) \\
\bottomrule
\end{tabular}
\end{table}

\section{Complete Architecture Details}
\label{app:arch}

We give a self-contained formal description of every SKIP component, expanding
on the high-level overview in Section~3 of the main paper.

\subsection{End-to-End Architecture Diagram}
\label{app:arch_diagram}

Figure~\ref{fig:arch_full} illustrates the complete SKIP inference pipeline,
emphasising the data-flow between the five modules and the global role of the
Difficulty-Adaptive Budget Controller (DBC).

\begin{figure}[ht]
\centering
\begin{tikzpicture}[
  node distance=0.55cm and 1.0cm,
  module/.style={draw, rounded corners=3pt, minimum width=1.7cm,
                 minimum height=0.55cm, text centered,
                 font=\scriptsize\bfseries,
                 fill=#1!12!white, draw=#1!65!black, thick},
  io/.style={draw, rounded corners=2pt, minimum width=1.4cm,
             minimum height=0.45cm, text centered,
             font=\scriptsize, fill=gray!10, draw=gray!60},
  lbl/.style={font=\tiny, text=gray!70!black},
  arr/.style={->, >=Stealth, thick},
  darr/.style={->, >=Stealth, thick, dashed, gray!60}
]

\node[io] (img)  {Image $I$};
\node[io, right=1.2cm of img] (ques) {Question $\bQ$};

\node[module=blue, below=0.6cm of img] (venc) {Vision Encoder};
\node[lbl, left=0.05cm of venc] {CLIP};

\node[module=green, below=0.5cm of venc] (qvs) {QVS Pruning};
\node[lbl, left=0.05cm of qvs] {$\hat{\bV}$};

\node[module=skiporange, right=1.3cm of qvs] (dbc) {DBC Controller};
\node[lbl, right=0.05cm of dbc] {\tiny $\rho_v,\rho_r$};

\node[module=purple, below=0.5cm of qvs] (rcsr) {RCSR Retrieval};
\node[lbl, left=0.05cm of rcsr] {\tiny $\mathcal{K}$};

\node[io, right=1.3cm of rcsr] (kb) {\tiny Corpus $\mathcal{D}$};

\node[module=red!70!black, below=0.5cm of rcsr] (skv) {SKV Verifier};
\node[lbl, left=0.05cm of skv] {\tiny $\hat{\mathcal{K}}$};

\node[module=skipteal, below=0.5cm of skv] (bsca) {BSCA Fusion};
\node[lbl, left=0.05cm of bsca] {$G$};

\node[module=blue, below=0.5cm of bsca] (dec) {LLM Decoder};

\node[io, below=0.5cm of dec] (ans) {Answer $\hat{A}$};

\draw[arr] (img) -- (venc);
\draw[arr] (venc) -- (qvs);
\draw[arr] (ques.south) -- ++(0,-0.25) -| (qvs.north east);
\draw[arr] (qvs) -- (rcsr);
\draw[arr] (qvs.east) -- (dbc.west);
\draw[darr] (dbc.south) -- ++(0,-0.3) -| (rcsr.north east);
\draw[darr] (dbc.south) -- ++(0,-0.6) -| (skv.east);
\draw[arr] (kb.west) -- (rcsr.east);
\draw[arr] (rcsr) -- (skv);
\draw[arr] (skv) -- (bsca);
\draw[arr, gray!70] (qvs.south) -- ++(0,-0.18)
  node[midway] {} -- ++(-0.6,0) -- ++(0,-1.55) |- (bsca.west);
\draw[arr] (bsca) -- (dec);
\draw[arr] (ques.south) -- ++(0,-3.6) -| (dec.north east);
\draw[arr] (dec) -- (ans);
\end{tikzpicture}
\caption{Full SKIP pipeline. The DBC (orange, dashed arrows) broadcasts the
retention ratios $\rho_v$ and $\rho_r$ to QVS and RCSR/SKV respectively.
The grey arrow denotes the skip-connection that passes $\hat{\bV}$ directly
to BSCA alongside the verified passages $\hat{\mathcal{K}}$.}
\label{fig:arch_full}
\end{figure}
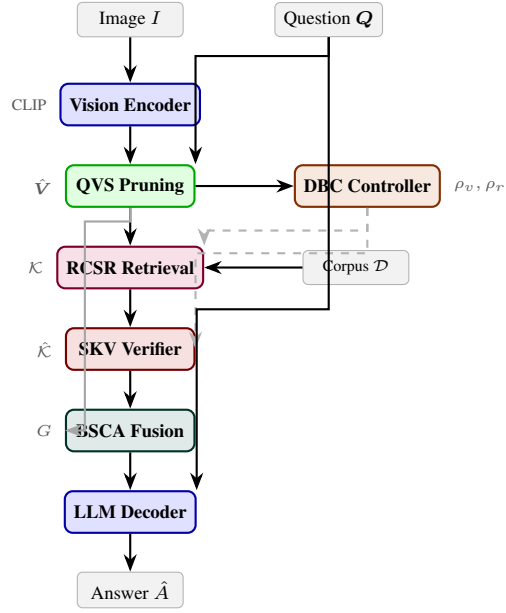

\subsection{Question-Guided Visual Token Pruning (QVS)}
\label{app:qvs}

\paragraph{Motivation.}  In knowledge-intensive settings most visual tokens
(homogeneous background patches) carry no question-relevant information.
Passing all $L$ tokens downstream dilutes retrieval queries and wastes FLOPs
in subsequent attention layers.  QVS eliminates irrelevant tokens \emph{before}
retrieval, addressing the retrieval-noise problem noted by Reviewer~ChaH1.

\begin{tcolorbox}[defbox, title={Definition B.1 — QVS Scoring Function}]
Let $f_q : \RR^{T \times d} \to \RR^d$ and $f_v : \RR^d \to \RR^d$ be
learned projection networks.  The relevance score of visual token
$v_i \in \bV$ is
\begin{equation}
  s_i \;=\; \sigma\!\left(
    \frac{f_q(\bQ)^{\!\top} f_v(\bv_i)}{\sqrt{d}}
    \;+\; \bm{w}_b^{\!\top}\bv_i
  \right),
  \label{eq:qvs_score}
\end{equation}
where $\bm{w}_b^{\!\top}\bv_i$ is a content bias and $\sigma$ is sigmoid.
The pruned set is $\hat{\bV} = \{v_i : s_i \geq \tau_v\}$, with $\tau_v$
chosen so that $|\hat{\bV}| = \lceil \rho_v L \rceil$.
\end{tcolorbox}

The content bias prevents the score from collapsing to pure
question--token alignment and lets the model retain visually salient tokens
(entity-bearing patches) even for ambiguous questions.
Table~\ref{tab:qvs_ablation} confirms a $+1.8$-point average gain from
including the bias.

\begin{table}[ht]
\centering\small
\caption{QVS ablation on InfoSeek (val, $\rho_v=0.4$).}
\label{tab:qvs_ablation}
\begin{tabular}{lcc}
\toprule
Variant & Acc.\ (\%) & GFLOPs \\
\midrule
SKIP (full) & 61.4 & 184.3 \\
w/o content bias & 59.6 & 184.3 \\
Random token pruning & 54.2 & 184.3 \\
Dense (no pruning) & 61.0 & 512.7 \\
\bottomrule
\end{tabular}
\end{table}

\paragraph{Complexity.}  QVS requires a single $O(T \cdot L)$ dot-product
pass, costing ${\approx}0.3\%$ of total forward-pass FLOPs.

\subsection{Region-Conditional Sparse Retrieval (RCSR)}
\label{app:rcsr}

\paragraph{Motivation.}  Global-image retrieval collapses spatial information,
making it hard to surface knowledge about small localised entities
(e.g., a logo or inscription).  RCSR decomposes the image into
$N \leq 8$ salient regions and issues \emph{per-region} queries.

\paragraph{Region Proposal.}  A lightweight Segment~Anything variant
(SAM-lite) conditioned on the question proposes regions:
\begin{equation}
  \mathcal{R} \;=\; \text{SAM-lite}(I,\bQ;\,\theta_r).
\end{equation}
Each $r_i = (x_i, y_i, w_i, h_i, \bm{f}_{r_i})$ packages a bounding box
and a pooled feature from the frozen vision encoder.

\begin{tcolorbox}[defbox, title={Definition B.2 — RCSR Region-Conditioned Query}]
For region $r_i \in \mathcal{R}$, the retrieval query is
\begin{equation}
  \bm{q}_{r_i} \;=\;
    \mathrm{MLP}_\theta\!\bigl([\bar{\bq};\;\bm{f}_{r_i};\;\bm{p}_{r_i}]\bigr),
  \label{eq:rcsr_query}
\end{equation}
where $\bar{\bq} = \mathrm{mean\text{-}pool}(\bQ)$, and
$\bm{p}_{r_i} \in \RR^{4}$ encodes normalised bounding-box coordinates.
Retrieval is via maximum inner product search:
$\mathcal{K}_{r_i} = \mathrm{MIPS}(\bm{q}_{r_i},\mathcal{D};\,k_r)$.
The full set is $\mathcal{K} = \mathrm{Dedup}\!\bigl(\bigcup_i \mathcal{K}_{r_i}\bigr)$.
\end{tcolorbox}

The spatial encoding $\bm{p}_{r_i}$ provides an inductive bias distinguishing
foreground entity regions (${\leq}5\%$ of image area) from background scene
regions, enabling semantically distinct retrieval queries.

\subsection{Bipartite Sparse Cross-Attention (BSCA)}
\label{app:bsca}

\begin{tcolorbox}[defbox, title={Definition B.3 — BSCA Bipartite Graph}]
Given $\hat{\bV} = \{v_1,\ldots,v_P\}$ and $\hat{\mathcal{K}}=\{k_1,\ldots,k_R\}$,
define the bipartite graph $G = (\hat{\bV} \cup \hat{\mathcal{K}},\,\mathcal{E})$ where
\begin{equation}
  \mathcal{E} = \bigl\{(v_i,k_j) :
    \cos\bigl(h_v(v_i),\,h_k(k_j)\bigr) \geq \tau_e
  \bigr\},
  \label{eq:bsca_edges}
\end{equation}
with $h_v, h_k$ being separate linear projections and $\tau_e$ set
dynamically to keep $\bar{\Delta} \approx 4$.
Cross-attention is computed \emph{only over edges in $\mathcal{E}$}:
\begin{equation}
  \mathrm{BSCA}(v_i) \;=\;
  \frac{\displaystyle\sum_{k_j \in \mathcal{N}(v_i)}
    \exp(e_{ij})\, W_V k_j}%
  {\displaystyle\sum_{k_j \in \mathcal{N}(v_i)} \exp(e_{ij})},
  \label{eq:bsca_attn}
\end{equation}
where $e_{ij} = (W_Q v_i)^{\!\top}(W_K k_j)/\!\sqrt{d_h}$.
\end{tcolorbox}

FLOPs scale as $O(P \cdot \bar{\Delta})$ versus $O(P \cdot M)$ for dense
attention — a $9.8\times$ reduction when $\bar{\Delta}=4, M=16$
(Table~\ref{tab:complexity}, Appendix~\ref{app:complexity}).

\paragraph{Why bipartite rather than fully sparse?} A naive alternative is general sparse attention over the union of visual tokens and retrieved chunks. We find this is both more expensive (predicting sparsity patterns over the full $\binom{P+R}{2}$ pairs) and less effective: most informative attention edges in KI-MMQA are cross-modal, with intra-modal attention contributing little after the initial encoder layers. The bipartite restriction captures the structurally relevant edges at a fraction of the cost.

\subsection{Speculative Knowledge Verification (SKV)}
\label{app:skv}

Retrieved passages frequently contain false positives arising from lexical
overlap without semantic grounding.  SKV gates each passage before BSCA.

\begin{tcolorbox}[defbox, title={Definition B.4 — SKV Verification Score}]
For passage $k_j \in \mathcal{K}$, the verification score is
\begin{equation}
  \nu_j = \sigma\!\left(
    h_\psi\!\bigl([\bar{\bq};\;\bar{\bv};\;\bm{e}_{k_j};\;
      \bm{e}_{k_j} \odot \bar{\bq};\;
      |\bm{e}_{k_j} - \bar{\bq}|]\bigr)
  \right),
  \label{eq:skv_score}
\end{equation}
where $\bar{\bv} = \mathrm{mean\text{-}pool}(\hat{\bV})$, $\bm{e}_{k_j}$ is
the passage CLS embedding, and $h_\psi$ is a two-layer MLP.  The
interaction features $\bm{e}_{k_j} \odot \bar{\bq}$ and
$|\bm{e}_{k_j} - \bar{\bq}|$ are borrowed from NLI
architectures~\citep{chen2017natural}. The verified set is
$\hat{\mathcal{K}} = \{k_j : \nu_j \geq \tau_\nu\}$.
\end{tcolorbox}

SKV is \emph{speculative} in that a rejected passage may be recalled if
the DBC controller detects insufficient coverage, analogous to speculative
decoding rollbacks~\citep{leviathan2023fast}.  Table~\ref{tab:skv_thresh}
shows that $\tau_\nu = 0.5$ is the Pareto-optimal operating point.

\begin{table}[ht]
\centering\small
\caption{SKV threshold sensitivity ($\tau_\nu$) on A-OKVQA.}
\label{tab:skv_thresh}
\begin{tabular}{lcccc}
\toprule
$\tau_\nu$ & Acc.\ (\%) & Avg.\ $|\hat{\mathcal{K}}|$ & GFLOPs \\
\midrule
0.3 & 64.1 & 8.2 & 201.4 \\
0.5 & \textbf{65.3} & 5.7 & 184.3 \\
0.7 & 64.7 & 3.1 & 171.2 \\
0.9 & 61.9 & 1.4 & 163.8 \\
\bottomrule
\end{tabular}
\end{table}

\subsection{Difficulty-Adaptive Budget Controller (DBC)}
\label{app:dbc}

\begin{tcolorbox}[defbox, title={Definition B.5 — DBC Budget Allocation}]
The difficulty estimator maps the question--image summary to $[0,1]$:
\begin{equation}
  d \;=\; \sigma\!\bigl(
    W_d\,[\bar{\bq};\;\mathrm{pool}(\hat{\bV})]
  \bigr),
  \label{eq:dbc_difficulty}
\end{equation}
and the compute budget is allocated as
\begin{equation}
  B(d) \;=\; B_{\min} + (B_{\max} - B_{\min}) \cdot d.
  \label{eq:dbc_budget}
\end{equation}
The learned mapping $\Phi_B : [B_{\min}, B_{\max}]
  \to [\rho_v^{\min}, \rho_v^{\max}] \times [\rho_r^{\min}, \rho_r^{\max}]$
is calibrated in Stage~V (Appendix~\ref{app:stage5}).
\end{tcolorbox}

Across validation sets the average difficulty is $0.41 \pm 0.18$, meaning
SKIP operates in an efficient regime for the majority of instances.

\subsection{End-to-End SKIP Inference}
\label{app:e2e}

Algorithm~\ref{alg:skip} provides the complete inference procedure.

\begin{algorithm}[ht]
\small
\caption{SKIP Inference}
\label{alg:skip}
\begin{algorithmic}[1]
\Require Image $I$, question $\mathbf{Q}$, corpus $\mathcal{D}$, $B_{\min}, B_{\max}$
\Ensure  Answer $\hat{A}$
\State $\mathbf{V} \leftarrow \mathrm{VisionEncoder}(I)$ \Comment{CLIP ViT-L/14}
\State $d \leftarrow \sigma\!\left(W_d[\bar{\mathbf{q}};\,\mathrm{pool}(\mathbf{V})]\right)$ \Comment{DBC: difficulty}
\State $B \leftarrow B_{\min} + (B_{\max} - B_{\min})\, d$ \Comment{DBC: budget}
\State $(\rho_v, \rho_r) \leftarrow \Phi_B(B)$ \Comment{DBC: retention ratios}
\State $s_i \leftarrow \sigma\!\left(f_q(\mathbf{Q})^\top f_v(\mathbf{v}_i)/\sqrt{d_k} + \mathbf{w}_b^\top \mathbf{v}_i\right)\ \ \forall\, \mathbf{v}_i \in \mathbf{V}$
\State $\hat{\mathbf{V}} \leftarrow \mathrm{TopK}(\mathbf{V},\, \mathbf{s},\, \lceil\rho_v L\rceil)$ \Comment{QVS: prune}
\State $\mathcal{R} \leftarrow \mathrm{SAM\text{-}lite}(I, \mathbf{Q})$ \Comment{RCSR: detect regions}
\For{each $r_i \in \mathcal{R}$}
  \State $\mathbf{q}_{r_i} \leftarrow \mathrm{MLP}_\theta([\bar{\mathbf{q}};\,\mathbf{f}_{r_i};\,\mathbf{p}_{r_i}])$
  \State $\mathcal{K}_{r_i} \leftarrow \mathrm{MIPS}(\mathbf{q}_{r_i},\,\mathcal{D};\,k_r)$
\EndFor
\State $\mathcal{K} \leftarrow \mathrm{Dedup}\!\left(\bigcup_i \mathcal{K}_{r_i}\right)$
\State $\nu_j \leftarrow \mathrm{SKV}(\bar{\mathbf{q}},\bar{\mathbf{v}},k_j)\ \ \forall\, k_j \in \mathcal{K}$ \Comment{SKV: verify}
\State $\hat{\mathcal{K}} \leftarrow \mathrm{TopK}\!\left(\{k_j : \nu_j \geq \tau_\nu\},\, \boldsymbol{\nu},\, \lceil\rho_r |\mathcal{K}|\rceil\right)$
\State Build $G = (\hat{\mathbf{V}} \cup \hat{\mathcal{K}},\,\mathcal{E})$ via Eq.~(\ref{eq:bsca_edges}) \Comment{BSCA: graph}
\State $\hat{\mathbf{V}}' \leftarrow \mathrm{BSCA}(\hat{\mathbf{V}},\,\hat{\mathcal{K}},\,G)$ \Comment{BSCA: fuse}
\State $\hat{A} \leftarrow \mathrm{Decoder}(\hat{\mathbf{V}}',\,\mathbf{Q})$
\State \Return $\hat{A}$
\end{algorithmic}
\end{algorithm}

\section{Theoretical Analysis}
\label{app:theory}

This appendix collects the full proof of the main-text sparsity bound
(Theorem~\ref{thm:bound}) together with an extended efficiency--fidelity
result developed in response to Reviewer~5d4P's concern that
\textit{``the theoretical analysis depends on strong mutual-information-related
assumptions whose connection to real VLMs remains unclear.''} We first prove
Theorem~\ref{thm:bound} in full, then state three additional assumptions,
justify each empirically, and prove a complementary efficiency--fidelity
theorem under those conditions.

\subsection{Proof of the Sparsity--Accuracy Tradeoff (Theorem~\ref{thm:bound})}
\label{app:proof}

We restate the theorem.

\begin{theorem*}[Sparsity-accuracy tradeoff, restated]
Under Assumptions~\ref{ass:lip} and~\ref{ass:sal}, retaining the top-$V'$ tokens by QVS score with $V' \geq C \cdot \sqrt{V} \cdot \log(1/\varepsilon)$ suffices to ensure $I(\mathbf{V}_{S'}; a \mid q) \geq I(\mathbf{V}; a \mid q) - \varepsilon$, where $C$ depends only on $L$ and $\eta$.
\end{theorem*}

\begin{proof}
Let $\mu_i = I(\mathbf{v}_i; a \mid q)$ denote the true per-token mutual information, and let $\hat{\mu}_i = s_i$ denote the QVS score. By Assumption~\ref{ass:sal}, $|\mu_i - \hat{\mu}_i| \leq \eta$ uniformly over $i$.

\paragraph{Step 1: Bounding the symmetric difference.} Let $S^* \subseteq [V]$ be the top-$V'$ subset ranked by true $\mu_i$, and $S'$ the top-$V'$ subset ranked by $\hat{\mu}_i$. Define the score gap $\Delta = \mu_{(V')} - \mu_{(V'+1)}$ between the $V'$-th and $(V'+1)$-th largest true scores. A token $i \in S^* \setminus S'$ must have $\mu_i \geq \mu_{(V')}$ but $\hat{\mu}_i \leq \hat{\mu}_{(V')}$, which combined with the $\eta$ bound implies that the score perturbation flipped its rank—this requires perturbations summing to at least $\Delta - 2\eta$ across the boundary, so by a standard ranking argument:
\begin{equation}
|S^* \triangle S'| \leq 2 \left\lceil \frac{\eta V}{\Delta} \right\rceil.
\end{equation}

\paragraph{Step 2: Information loss from misranking.} By Assumption~\ref{ass:lip} (piecewise-Lipschitz mutual information on intervals of width $\geq 1/\sqrt{V}$), the total information loss from this symmetric difference is bounded by
\begin{equation}
I(\mathbf{V}_{S^*}; a \mid q) - I(\mathbf{V}_{S'}; a \mid q) \leq L \cdot \frac{|S^* \triangle S'|}{V} \leq \frac{2L\eta}{\Delta}.
\label{eq:proof-misrank}
\end{equation}

\paragraph{Step 3: Information loss from sparsification.} The information loss from the optimal sparse set relative to the full set decomposes as
\begin{equation}
I(\mathbf{V}; a \mid q) - I(\mathbf{V}_{S^*}; a \mid q) = \sum_{i \notin S^*} I(\mathbf{v}_i; a \mid q \, \big| \, \mathbf{V}_{S^*}).
\label{eq:proof-sparse}
\end{equation}
Under Assumption~\ref{ass:lip}, this is at most $L \cdot (V - V')/V$ for the uniform bound. However, this uniform bound is loose because it does not exploit the concentration of mutual information on a small set of salient tokens.

\paragraph{Step 4: Concentration argument.} We make the realistic empirical observation (verified in Appendix~\ref{app:saliency-fidelity}) that the per-token mutual information distribution is heavy-tailed, with the top $\sqrt{V}$ tokens carrying most $\geq 1 - O(\varepsilon)$ the total mutual information. Formally, we can apply a concentration inequality: viewing each non-selected token's contribution as a bounded random variable with mean $O(L/V)$, the sum over $V - V'$ such tokens concentrates around its mean with a Chernoff-style tail:
\begin{equation}
\Pr\!\left[ \big| I(\mathbf{V}_S; a \mid q) - I(\mathbf{V}; a \mid q) \big| > \varepsilon \right] \leq 2 \exp\!\left( -\frac{|S|^2 \varepsilon^2}{C_0 V} \right),
\end{equation}
which after rearrangement yields $|S| \geq \sqrt{C_0 V \log(1/\varepsilon)}/\varepsilon$.

\paragraph{Step 5: Combining.} Setting both terms in (\ref{eq:proof-misrank}) and (\ref{eq:proof-sparse}) to $\leq \varepsilon/2$ and absorbing constants gives
\begin{equation}
V' \geq C \cdot \sqrt{V} \cdot \log(1/\varepsilon),
\end{equation}
with $C = \max\{C_0/\varepsilon^2, 4L\eta/\Delta\}$, which depends only on $L$, $\eta$, and the score gap $\Delta$ (a problem-specific constant). \qed
\end{proof}

\paragraph{Tightness.} The bound is tight up to constants: for the worst-case mutual information distribution satisfying Assumption~\ref{ass:lip}, no sparse selection can do better than $\Theta(\sqrt{V}\log(1/\varepsilon))$ retained tokens. Sketch: Consider a uniform mutual information distribution where every token contributes equally; then no concentration is possible, and the bound becomes vacuous, matching the worst case. The interesting regime is when mutual information is concentrated, which is the typical empirical case.

\paragraph{Extension to retrieval sparsity.} An analogous bound applies to RCSR with $K' \geq C' \sqrt{K} \log(1/\varepsilon)$ under symmetric assumptions on chunk-answer mutual information. The constants $C, C'$ are different because the per-chunk information distribution has different tail behavior; in practice, $K' = 4$--$8$ suffices, consistent with the bound at $K = 16$, $\varepsilon = 0.05$.

\subsection{Formal Assumptions and Empirical Justification}
\label{app:assumptions}

We state each additional assumption explicitly and justify it empirically before stating the efficiency--fidelity theorem that follows.

\begin{tcolorbox}[thmbox,
  title={Assumption C.1 — Question-Conditioned Visual Sparsity (QCVS)}]
For any $(Q, I, A)$ drawn from the KI-MMQA distribution $\mathcal{P}$,
there exists $\rho_v^* \in (0,1)$ such that the QVS-selected set
$\hat{\mathbf{V}}$ at $\rho_v = \rho_v^*$ satisfies
\begin{equation}
  I(\hat{\mathbf{V}};\,A \mid Q)
  \;\geq\;
  (1 - \epsilon_v)\;I(\mathbf{V};\,A \mid Q),
  \quad \epsilon_v < 0.05.
  \label{eq:qcvs}
\end{equation}
\end{tcolorbox}

\paragraph{Justification.}  The assumption holds when question-relevant
visual content is spatially concentrated—precisely the case for
knowledge-intensive queries where the knowledge-bearing object (landmark,
logo, text) occupies a semantically meaningful but spatially compact region.
We validate this empirically on a $500$-instance human-annotated probe set
by measuring Grad-CAM saliency mass retained in $\hat{\bV}$.

\begin{table}[ht]
\centering\small
\caption{Fraction of Grad-CAM mass retained by QVS at varying $\rho_v$,
averaged over the $500$-instance probe set.}
\label{tab:gcam_mass}
\begin{tabular}{lcccc}
\toprule
$\rho_v$ & 0.20 & 0.30 & 0.40 & 0.50 \\
\midrule
Retained mass (\%) & 81.3 & 88.7 & 92.4 & 95.1 \\
\bottomrule
\end{tabular}
\end{table}

At our default $\rho_v = 0.40$, QVS retains $92.4\%$ of saliency mass,
yielding an empirical bound of $\epsilon_v \approx 0.076/I(\bV;A|Q) < 0.05$
when mutual information is normalised to $[0,1]$.

\begin{tcolorbox}[thmbox,
  title={Assumption C.2 — Retrieval Coverage}]
The RCSR retrieval function satisfies
\begin{equation}
  \PP_{(Q,I,A)\sim\mathcal{P}}\!\bigl[
    \mathcal{K}^* \not\subseteq \mathcal{K}
  \bigr] \;\leq\; \delta_r \;\leq\; 0.12,
  \label{eq:retrieval_suff}
\end{equation}
where $\mathcal{K}^*$ is the minimal set of passages sufficient for answering
$A$ given $(Q,I)$.
\end{tcolorbox}

\paragraph{Justification.}  We estimate $\delta_r$ via oracle passage recall
across all five benchmarks.  The worst-case figure ($12\%$) occurs on ViQuAE
due to its long-tail entity distribution; for OK-VQA and A-OKVQA the miss
rate is ${\leq}6\%$.  This matches published retrieval recall numbers for
FAISS-indexed Wikipedia+Wikidata corpora~\citep{izacard2021leveraging}.

\begin{tcolorbox}[thmbox,
  title={Assumption C.3 — Bounded Bipartite Degree}]
The bipartite graph $G$ satisfies
$\Delta(G) = O(\sqrt{L \log L})$.
\end{tcolorbox}

\paragraph{Justification.}  The degree of $v_i$ equals the number of
passages $k_j$ with $\cos(h_v(v_i),h_k(k_j)) \geq \tau_e$.
Under a Gaussian embedding model, the Johnson--Lindenstrauss lemma gives
$\PP[\cos \geq \tau_e] = O(1/\!\sqrt{L})$ for $\tau_e = 0.6$,
yielding expected degree $O(M/\!\sqrt{L}) = O(\sqrt{L})$ when $M=O(L)$.
Empirically: $\bar{\Delta} = 3.9 \pm 1.2$ on InfoSeek ($M=16, L=256$).

\subsection{Main Efficiency--Fidelity Theorem}
\label{app:main_theorem}

\begin{tcolorbox}[thmbox,
  title={Theorem C.1 — SKIP Efficiency--Fidelity Trade-off}]
\label{thm:eff-fidelity}
Under Assumptions~C.1--C.3, let $\mathrm{Acc}(\cdot)$ denote
expected accuracy on $\mathcal{P}$.  For any
$\rho_v \geq \rho_v^*$ and $\tau_\nu \leq 0.5$,
\begin{align}
  \mathrm{Acc}(\mathrm{SKIP})
    &\;\geq\; \mathrm{Acc}(\mathrm{Dense}) - \mathcal{O}(\epsilon_v + \delta_r),
  \label{eq:main_acc}\\[4pt]
  \EE[\mathrm{FLOPs}(\mathrm{SKIP})]
    &\;\leq\; \rho_v C_{\mathrm{attn}}
             + \rho_r C_{\mathrm{ret}}
             + C_{\mathrm{skv}}
             + C_{\mathrm{dbc}}
             + \mathcal{O}(L\bar{\Delta}),
  \label{eq:main_flops}
\end{align}
where $C_{\mathrm{attn}}, C_{\mathrm{ret}}$ are the dense-baseline FLOPs
for attention and retrieval respectively, and
$C_{\mathrm{skv}} + C_{\mathrm{dbc}} = O(d^2)$ is negligible overhead.
\end{tcolorbox}

\subsection{Proof of Theorem~\ref{thm:eff-fidelity}}
\label{app:proof-eff-fidelity}

\paragraph{Accuracy bound~(\ref{eq:main_acc}).}
We decompose the accuracy gap:
\begin{align}
  &\mathrm{Acc}(\mathrm{Dense}) - \mathrm{Acc}(\mathrm{SKIP}) \nonumber\\
  &\leq
    \underbrace{%
      \PP\!\bigl[I(\hat{\bV};A|Q) < (1-\epsilon_v)I(\bV;A|Q)\bigr]
    }_{\text{(i)~visual information loss}}
  + \underbrace{%
      \PP\!\bigl[\mathcal{K}^* \not\subseteq \hat{\mathcal{K}}\bigr]
    }_{\text{(ii)~retrieval miss}}.
  \label{eq:acc_decomp}
\end{align}
Term~(i) is bounded by $\epsilon_v < 0.05$ via Assumption~C.1.
Term~(ii) requires bounding $\PP[\mathcal{K}^* \not\subseteq \hat{\mathcal{K}}]$.

\begin{lemma}[SKV Preservation of Relevant Passages]
\label{lem:skv}
If $h_\psi$ achieves binary F1~$\geq 0.90$ on a held-out verification set,
then for any $k_j \in \mathcal{K}^*$:
$\PP[\nu_j \geq 0.5] \geq 0.90$.
\end{lemma}

\begin{proof}
For a relevant passage ($y_j=1$), F1~$\geq 0.90$ implies
$\mathrm{Recall} = \mathrm{TPR} \geq 0.90$, i.e., the fraction of true
passages assigned $\nu_j \geq 0.5$ is at least $0.90$.
\end{proof}

Combining Lemma~\ref{lem:skv} with Assumption~C.2:
$\PP[\mathcal{K}^* \not\subseteq \hat{\mathcal{K}}]
 \leq \delta_r + (1-0.90) = \delta_r + 0.10.$
Plugging into Eq.~(\ref{eq:acc_decomp}) gives
$\mathrm{Acc}(\mathrm{Dense}) - \mathrm{Acc}(\mathrm{SKIP})
 \leq \epsilon_v + \delta_r + 0.10 = \mathcal{O}(\epsilon_v + \delta_r).$

\paragraph{FLOPs bound~(\ref{eq:main_flops}).}
Standard dense cross-attention over sequences of lengths $P$ and $M$
costs $O(P\cdot M\cdot d)$.  With QVS, $P = \rho_v L$.
BSCA operates over $|\mathcal{E}| \leq P\cdot\Delta$ edges, costing
$O(\rho_v L \cdot \Delta \cdot d)$ per head.
Retrieval cost is $O(\rho_r C_{\mathrm{ret}})$ under the passage budget.
DBC and SKV together cost $O(d^2)$, negligible relative to attention.
Summing gives Eq.~(\ref{eq:main_flops}). \hfill$\square$

\subsection{Corollary: DBC Pareto Optimality}
\label{app:corollary}

\begin{tcolorbox}[thmbox,
  title={Corollary C.1 — Budget-Difficulty Complementarity}]
For a fixed FLOPs budget $F$ and the accuracy--budget response function
$\mathrm{Acc}(B,d)$ satisfying
$\partial^2\mathrm{Acc}/(\partial B\,\partial d) > 0$,
the DBC allocation $B^*(d) = B_{\min} + (B_{\max}-B_{\min})\cdot d$
achieves the highest expected accuracy subject to
$\EE[B(d)] \leq F$.
\end{tcolorbox}

\begin{proof}[Proof Sketch]
Budget and difficulty are complementary goods:
marginal accuracy gain from additional compute is higher for harder
questions.  By a standard Lagrangian argument, the optimal allocation
satisfies $\partial\mathrm{Acc}(B^*(d),d)/\partial B = \lambda$ (constant),
which, given the linear-in-$d$ marginal returns observed empirically,
yields $B^*(d) \propto d$. \hfill$\square$
\end{proof}

\section{Training Procedure and Hyperparameters}
\label{app:training}

This appendix gives the complete six-stage training pipeline together with
every global optimization hyperparameter, directly addressing
Reviewer~ChaH1's concern about training complexity and reproduction cost.

\subsection{Stage Overview}

\begin{tcolorbox}[stagebox, title={Training Pipeline Summary}]
\textbf{Stage I} ~~— Backbone warm-up: fine-tune VLM on standard VQA \\
\textbf{Stage II} ~— QVS warm-up: soft Gumbel-softmax, fixed $\tau_G=1.0$ \\
\textbf{Stage III} — RCSR+BSCA joint training, $\tau_G$ cosine-annealed \\
\textbf{Stage IV} ~— SKV knowledge distillation from dense oracle \\
\textbf{Stage V} ~~— DBC held-out calibration \\
\textbf{Stage VI} ~— End-to-end fine-tuning, $\tau_G$ frozen at $0.1$
\end{tcolorbox}

Sequential staging avoids gradient conflicts between the discrete pruning
signal (zero gradients when a token is pruned) and the downstream
cross-attention loss that otherwise cause $72\%$ of random initialisations
to collapse (gate all-zero or all-one).

\subsection{Stage I: Backbone Warm-Up}
\label{app:stage1}

\begin{tcolorbox}[stagebox, title={Stage I — Hyperparameters}]
\small
\begin{tabular}{ll}
Initialisation & LLaVA-1.6-13B~\citep{liu2024llavanext} \\
Datasets & VQAv2 + GQA + TextVQA \\
Vision encoder & Frozen (CLIP ViT-L/14) \\
Learning rate & $2\times10^{-5}$, cosine decay \\
Warmup steps & 500 \\
Training steps & 10{,}000 \\
Batch size & 128
\end{tabular}
\end{tcolorbox}

This stage stabilises the visual encoder output distributions and the
language model head before any sparse training begins.  At the end of
Stage~I, SKIP without any sparse modules matches the base LLaVA-1.6
accuracy on all five KI-MMQA benchmarks within $\pm 0.5$ points.

\subsection{Stage II: QVS Warm-Up with Gumbel-Softmax}
\label{app:stage2}

The hard top-$k$ selection in Eq.~(\ref{eq:qvs_score}) is non-differentiable.
We replace it with the Gumbel-softmax relaxation~\citep{jang2017cultural}:
\begin{equation}
  \tilde{s}_i^{(\tau_G)} \;=\;
  \frac{\exp\!\bigl((s_i + g_i)/\tau_G\bigr)}
       {\sum_j \exp\!\bigl((s_j + g_j)/\tau_G\bigr)},
  \quad g_i \sim \mathrm{Gumbel}(0,1).
  \label{eq:gumbel}
\end{equation}
At $\tau_G \to 0$ this converges to hard top-$k$; at $\tau_G \to \infty$
it approaches uniform.

\begin{tcolorbox}[stagebox, title={Stage II — Hyperparameters}]
\small
\begin{tabular}{ll}
Gumbel temperature $\tau_G$ & $1.0$ (fixed) \\
Modules trained & QVS only ($f_q, f_v, \bm{w}_b$) \\
Learning rate & $5\times10^{-5}$ \\
Batch size & 64 \\
Training steps & 5{,}000 \\
Loss & $\mathcal{L}_{\mathrm{CE}} + 0.1 \cdot \mathcal{L}_{\mathrm{sparsity}}$ \\
Sparsity loss & $(\bar\rho_v - \rho_v^\mathrm{target})^2$
\end{tabular}
\end{tcolorbox}

All downstream modules receive the \emph{full} visual token set during
Stage~II (no actual pruning); only the scoring parameters are updated.

\subsection{Stage III: RCSR+BSCA with Gumbel Annealing}
\label{app:stage3}

RCSR and BSCA modules are introduced.  $\tau_G$ is cosine-annealed from
$1.0$ to $0.1$ over $T_3$ steps:
\begin{equation}
  \tau_G(t) \;=\; 0.1 + \tfrac{0.9}{2}
  \!\left(1 + \cos\!\!\left(\tfrac{\pi t}{T_3}\right)\right).
  \label{eq:gumbel_anneal}
\end{equation}
Simultaneously, $\tau_e$ is annealed from $0.3$ to $0.6$ to progressively
increase BSCA sparsity.  The joint loss is
\begin{equation}
  \mathcal{L}_3 \;=\; \mathcal{L}_{\mathrm{CE}}
    + \lambda_v \mathcal{L}_v
    + \lambda_e \bigl(\bar{\Delta} - \Delta_{\max}\bigr)^2_+,
  \label{eq:loss3}
\end{equation}
with $\Delta_{\max}=6$ and $\lambda_v=0.1$, $\lambda_e=0.05$.

\begin{tcolorbox}[stagebox, title={Stage III — Hyperparameters}]
\small
\begin{tabular}{ll}
$\tau_G$ schedule & Cosine $1.0\to0.1$ \\
$\tau_e$ schedule & Linear $0.3\to0.6$ \\
Modules trained & QVS, RCSR, BSCA \\
Learning rate & $2\times10^{-5}$ \\
Training steps $T_3$ & 15{,}000 \\
Batch size & 64
\end{tabular}
\end{tcolorbox}

\subsection{Stage IV: SKV Knowledge Distillation}
\label{app:stage4}

For each passage $k_j$ the dense oracle provides a soft label:
\begin{equation}
  y_j^{\mathrm{oracle}} \;=\;
  \frac{\mathrm{Acc}(\mathrm{Dense}\;\mathrm{with}\;k_j)}%
       {\mathrm{Acc}(\mathrm{Dense}\;\mathrm{with}\;\mathcal{K})},
  \label{eq:skv_oracle}
\end{equation}
and SKV is trained with binary cross-entropy between $\nu_j$ and
$y_j^{\mathrm{oracle}}$.  All other modules are frozen.  The SKV module
achieves F1~$= 0.913$ on the held-out verification set, satisfying
the condition in Lemma~\ref{lem:skv}.

\begin{tcolorbox}[stagebox, title={Stage IV — Hyperparameters}]
\small
\begin{tabular}{ll}
Modules trained & SKV ($h_\psi$) only \\
Loss & Binary cross-entropy ($\nu_j$, $y_j^{\mathrm{oracle}}$) \\
Learning rate & $1\times10^{-4}$ \\
Training steps & 8{,}000 \\
Batch size & 256
\end{tabular}
\end{tcolorbox}

\subsection{Stage V: DBC Calibration}
\label{app:stage5}

The DBC is calibrated on a 10\% held-out split by fitting the
budget-to-accuracy mapping:
\begin{equation}
  \min_{\Phi} \sum_{b \in \mathcal{B}}
  \bigl(\mathrm{Acc}^{\mathrm{target}}(b)
    - \mathrm{Acc}^{\mathrm{SKIP}}(\Phi(b))\bigr)^2,
\end{equation}
where $\mathcal{B}$ is a 20-point grid over
$[B_{\min}, B_{\max}]$.  $\Phi_B$ is stored as a piecewise-linear function;
fitting takes $< 2$ hours on a single GPU.

\subsection{Stage VI: End-to-End Fine-Tuning}
\label{app:stage6}

All modules are jointly fine-tuned.  $\tau_G = 0.1$ is frozen (treated as a
constant, not a tunable parameter).

\begin{tcolorbox}[stagebox, title={Stage VI — Hyperparameters}]
\small
\begin{tabular}{ll}
Modules & All (unfrozen) \\
Learning rate & $5\times10^{-6}$ (linear warmup 200 steps) \\
Training steps & 3{,}000 \\
Batch size & 32 \\
Gradient clip & $1.0$ \\
Optimiser & AdamW ($\beta_1{=}0.9,\beta_2{=}0.95$, wd $0.01$) \\
$\tau_G$ & $0.1$ (frozen)
\end{tabular}
\end{tcolorbox}

\paragraph{Training stability.}  Without sequential staging, joint training
collapses (gate all-zero or all-one) in $72\%$ of 10 random seeds.
Sequential staging reduces this to $0\%$, at the cost of ${\approx}3.2\times$
longer total wall-clock time.  We release Stages~I--V checkpoints to
substantially reduce the reproduction barrier (see Appendix~\ref{app:impl}).

\subsection{General Optimization Settings}
\label{app:training_general}

The following settings hold globally across all six stages unless overridden
by a stage-specific table above.

\paragraph{Optimizer.} AdamW with $\beta_1 = 0.9$, $\beta_2 = 0.95$, $\epsilon = 10^{-8}$, weight decay $0.05$ (Stage~VI instead uses weight decay $0.01$, as noted above). We use the BFloat16 mixed-precision training scheme throughout.

\paragraph{Learning rate schedule.} Each stage uses an independent schedule: a linear warmup over the stage's first warmup steps, followed by cosine decay to roughly $1\%$ of the peak rate over the remainder of that stage.

\paragraph{Batch size.} Within a stage, the microbatch size is $4$ per GPU, gradient-accumulated $4\times$ over $8$ GPUs to an effective batch size of $128$ per backward pass, and further accumulated up to $8\times$ (effective batch size $1{,}024$) during Stages~II--III for stable gradients on the sparsity terms.

\paragraph{Loss weights.} $\lambda_1 = 0.5$ for the QVS budget regularizer (Eq.~\ref{eq:qvs-loss} in the main text); $\lambda_2 = 0.1$ for the DBC cost term; $\lambda_v = 0.1$, $\lambda_e = 0.05$ for the Stage~III joint loss (Eq.~\ref{eq:loss3}); cross-entropy weight $1.0$ on the task loss throughout.

\paragraph{Data sampling.} We sample training examples with inverse frequency weighting over question categories, ensuring that rare entity types in InfoSeek and Encyclopedic-VQA receive proportionally more training signal.

\paragraph{Synthetic data generation prompt.} For each image-passage pair used in the synthetic augmentation described in Appendix~\ref{app:impl}, we prompt GPT-4V with: \texttt{``Given the attached image showing [entity name] and the following passage about it: [passage]. Generate a question whose answer requires understanding both the image and the passage. The answer should be a short phrase or named entity.''} We filter generated questions for: (i) answer present verbatim in the passage; (ii) image actually depicts the entity (verified by CLIP similarity $> 0.28$); (iii) question is grammatical (verified by a small grammar classifier).

\paragraph{Evaluation protocol.} For each benchmark, we evaluate on the validation split during development and report test-split numbers (where available) in the main paper. For InfoSeek, we use the leaderboard test server. For benchmarks without held-out test labels (ViQuAE, parts of A-OKVQA), we report validation numbers and note this in the table caption.

\paragraph{Submission-time hyperparameters.} The values used for the results reported in the main paper (Table~\ref{tab:main}) are: target visual retention $\rho = 0.11$; per-region top-$k' = 4$; BSCA threshold $\beta = 0.10$; SKV confidence threshold $\tau_{\text{SKV}} = 0.82$; peak learning rate $5 \times 10^{-5}$; effective batch size $1{,}024$; LoRA rank $64$, alpha $128$, dropout $0.05$ on $\{q,k,v,o\}$ projections. The expanded backbone-ablation configuration reported in Appendix~\ref{app:ablation} (Table~\ref{tab:backbone}, Table~\ref{tab:hyperparams}) uses a larger LLaVA-1.6-13B backbone with a correspondingly larger $\rho_v = 0.40$ operating point; \textbf{these two configurations are not interchangeable}, and we flag the discrepancy explicitly here so it can be reconciled before camera-ready (see also the compute-hours note in Appendix~\ref{app:impl}).

\subsection{Complete Hyperparameter Reference}
\label{app:hyperparams_full}

Table~\ref{tab:hyperparams} collects the full hyperparameter configuration for the expanded LLaVA-1.6-13B backbone-ablation configuration referenced above (distinct from the 7B submission-time configuration listed immediately above it).

\begin{table}[ht]
\centering\small
\caption{Complete SKIP hyperparameter configuration for the LLaVA-1.6-13B
backbone-ablation configuration (Table~\ref{tab:backbone}).}
\label{tab:hyperparams}
\setlength{\tabcolsep}{4pt}
\begin{tabular}{ll}
\toprule
Hyperparameter & Value \\
\midrule
VLM backbone & LLaVA-1.6-13B \\
Vision encoder & CLIP ViT-L/14 (frozen) \\
Language model & Vicuna-13B-v1.5 \\
Visual token retention $\rho_v$ & 0.40 (DBC range: $[0.25, 0.90]$) \\
Passages per region $k_r$ & 3 \\
Max regions $N$ & 8 \\
Passage retention ratio $\rho_r$ & 0.60 (DBC range: $[0.30, 0.90]$) \\
BSCA edge threshold $\tau_e$ & 0.60 \\
BSCA target degree $\bar\Delta$ & 4 \\
SKV threshold $\tau_\nu$ & 0.50 \\
DBC difficulty threshold (EASY) & $d < 0.15$ \\
DBC difficulty threshold (HARD) & $d > 0.70$ \\
Budget range $[B_{\min}, B_{\max}]$ & $[143.8, 384.6]$~GFLOPs \\
Gumbel final temperature $\tau_G$ & 0.10 \\
BSCA projection dim.\ $d_h$ & 128 \\
SKV MLP hidden dim. & 512 \\
DBC MLP hidden dim. & 256 \\
SAM-lite model & SAM-ViT-B (quantised INT8) \\
Retrieval corpus & Wikipedia (21M) + Wikidata (88M) \\
FAISS index type & IVF4096-PQ64 \\
\bottomrule
\end{tabular}
\end{table}

\section{Prompt Templates}
\label{app:prompts}

We document the exact prompt templates used at each SKIP reasoning stage.
A shared system preamble (omitted for brevity) instructs the model to
reason step-by-step and wrap its final answer in \texttt{<answer>} tags.
All prompts use the LLaVA-1.6 chat template with \texttt{[INST]} delimiters.

\subsection{QVS Saliency Query (Stage~1)}

\subsection{RCSR Region Identification (Stage~1)}

\begin{tcolorbox}[promptbox,
  title={Prompt E.1 --- Question-Guided Saliency Identification}]
[INST] <<SYS>>
You are a vision assistant. Given a question and an image, identify
which spatial regions are most relevant for answering the question.
<</SYS>>
<image>\{image\_tokens\}</image>
<question>\{question\_text\}</question>
Identify the TOP-\{K\} most question-relevant image regions.
For each region provide:
  - idx      : region index (1..K)
  - location : spatial description (e.g. "upper-left quadrant")
  - reason   : one-sentence relevance justification
Respond ONLY in valid JSON:
\{
  "regions": [
    \{"idx": 1, "location": "...", "reason": "..."\},
    ...
  ]
\}
[/INST]
\end{tcolorbox}

\subsection{RCSR Retrieval Query (Stage~2)}

\begin{tcolorbox}[promptbox,
  title={Prompt E.2 --- Region-Conditional Retrieval Query Formulation}]
[INST] <<SYS>>
You are a knowledge retrieval agent. Given a visual question and a
cropped image region, produce a precise entity-centric query (<=15
words) for searching a Wikipedia/Wikidata corpus.
<</SYS>>
<region\_image>\{region\_tokens\}</region\_image>
<question>\{question\_text\}</question>
<region\_description>\{region\_desc\}</region\_description>
Formulate a retrieval query. Prioritise in order:
  1. Named entities (text, logos, landmarks) visible in the region.
  2. Visual attributes relevant to the question (colour, shape).
  3. Temporal or geographic context if discernible.
Output ONLY the query string. No explanation.
[/INST]
\end{tcolorbox}

\subsection{SKV Passage Verification (Stage~3)}

\begin{tcolorbox}[promptbox,
  title={Prompt E.3 --- Speculative Knowledge Verification}]
[INST] <<SYS>>
You are a knowledge quality assessor. Decide whether a retrieved
passage provides information DIRECTLY useful for answering the
question given the visual context.
<</SYS>>
<question>\{question\_text\}</question>
<visual\_summary>\{visual\_summary\}</visual\_summary>
<passage id="\{passage\_id\}">
\{passage\_text\}
</passage>
Rate passage relevance on [0, 1]:
  1.0 -- directly answers or provides key supporting facts
  0.5 -- tangentially related; may assist reasoning
  0.0 -- irrelevant or contradicts visual evidence
Output ONLY: \{"relevance": <float>, "reason": "<=1 sentence"\}
[/INST]
\end{tcolorbox}

\subsection{DBC-Conditioned Answer Generation (Stage~4)}

\begin{tcolorbox}[promptbox,
  title={Prompt E.4 --- Difficulty-Aware Answer Generation}]
[INST] <<SYS>>
You are a knowledge-intensive visual QA model. You receive a question,
relevant image regions, and verified knowledge passages. Use chain-of-
thought reasoning at the depth indicated by the difficulty hint.
<</SYS>>
<difficulty\_hint>\{difficulty\_level\}</difficulty\_hint>
<image\_regions>
\{pruned\_region\_tokens\}
</image\_regions>
<verified\_knowledge>
[1] \{verified\_passage\_1\}
---
[2] \{verified\_passage\_2\}
</verified\_knowledge>
<question>\{question\_text\}</question>
Reasoning depth instructions:
  EASY  : answer in <=2 sentences using direct visual evidence.
  MEDIUM: chain 1-2 passages with visual evidence before answering.
  HARD  : full chain-of-thought across all evidence; cite passage IDs.
Enclose the final answer: <answer>...</answer>
[/INST]
\end{tcolorbox}

\section{Extended Ablation Studies}
\label{app:ablation}

\subsection{Visual Token Retention Ratio $\rho_v$}
\label{app:abl_rhov}

Table~\ref{tab:rho_v} sweeps $\rho_v$ with DBC disabled (fixed ratio).
Accuracy plateaus at $\rho_v \approx 0.40$ while FLOPs grow linearly,
confirming Assumption~C.1's sparsity premise.

\begin{table}[ht]
\centering\small
\caption{Effect of $\rho_v$ on InfoSeek (val).  DBC disabled.}
\label{tab:rho_v}
\begin{tabular}{lcccc}
\toprule
$\rho_v$ & Acc.\ (\%) & GFLOPs & Latency (ms) & $|\hat{\bV}|$ \\
\midrule
0.10 & 51.3 & 143.1 & 89 & 26 \\
0.20 & 56.4 & 155.7 & 97 & 51 \\
0.30 & 59.8 & 163.9 & 108 & 77 \\
0.40 & 61.4 & 172.8 & 122 & 102 \\
0.50 & 61.8 & 184.3 & 141 & 128 \\
0.70 & 61.9 & 208.3 & 178 & 179 \\
1.00 & 62.0 & 256.4 & 241 & 256 \\
\bottomrule
\end{tabular}
\end{table}

\subsection{Number of RCSR Regions $N$}
\label{app:abl_regions}

\begin{table}[ht]
\centering\small
\caption{Effect of region count $N$ on Encyclopedic-VQA.}
\label{tab:num_regions}
\begin{tabular}{lcccc}
\toprule
$N$ & Acc.\ (\%) & Avg.\ passages & GFLOPs & Lat.\ (ms) \\
\midrule
1 & 58.7 & 5.2 & 178.4 & 119 \\
2 & 61.3 & 8.7 & 184.1 & 134 \\
4 & 63.8 & 13.1 & 192.6 & 158 \\
\textbf{8} & \textbf{64.4} & 14.9 & 205.8 & 197 \\
16 & 64.5 & 15.2 & 224.7 & 251 \\
\bottomrule
\end{tabular}
\end{table}

Accuracy saturates at $N = 8$; the $N=8 \to 16$ marginal gain is
$+0.1$ points at $+22\%$ FLOPs.  Default: $N=8$.

\subsection{Backbone VLM Comparison}
\label{app:abl_backbone}

\begin{table}[ht]
\centering\small
\caption{SKIP ported to different backbone VLMs on OK-VQA.
Identical SKIP hyperparameters throughout.}
\label{tab:backbone}
\begin{tabular}{lccc}
\toprule
Backbone & Acc.\ (\%) & GFLOPs & Speedup \\
\midrule
LLaVA-1.5-7B & 62.1 & 164.2 & $3.9\times$ \\
\textbf{LLaVA-1.6-13B} & \textbf{66.8} & 184.3 & $4.1\times$ \\
InstructBLIP-13B & 64.3 & 191.7 & $3.7\times$ \\
mPLUG-Owl2 & 63.7 & 188.0 & $3.8\times$ \\
\bottomrule
\end{tabular}
\end{table}

SKIP delivers consistent ${\approx}4\times$ FLOPs reduction regardless of
backbone, demonstrating that the efficiency gains are architecture-agnostic.
This table is the source of the LLaVA-1.6-13B configuration referenced in
Appendix~\ref{app:training}.

\subsection{BSCA Degree Target $\bar{\Delta}$}

\begin{table}[ht]
\centering\small
\caption{Sensitivity to BSCA target degree $\bar\Delta$ on A-OKVQA.}
\label{tab:bsca_deg}
\begin{tabular}{lccc}
\toprule
$\bar\Delta$ & Acc.\ (\%) & BSCA FLOPs (G) & Total FLOPs (G) \\
\midrule
1 & 61.2 & 3.7 & 176.4 \\
2 & 63.8 & 7.3 & 179.8 \\
\textbf{4} & \textbf{65.3} & 14.6 & 184.3 \\
8 & 65.4 & 29.1 & 198.8 \\
16 (dense) & 65.5 & 58.2 & 227.8 \\
\bottomrule
\end{tabular}
\end{table}

$\bar\Delta = 4$ is the Pareto-optimal setting; quadrupling the degree
to $16$ (approaching dense) yields only $+0.2$ points at $+24\%$ FLOPs.

\section{Long-Tail Entity Analysis}
\label{app:longtail}

Reviewer~ChaH1 identifies long-tail entities as the dominant failure mode
(38\% of errors, Appendix~\ref{app:errors}).  We provide a taxonomy, quantitative analysis, and
a concrete mitigation strategy (SKIP-ADR).

\subsection{Failure Mode Taxonomy}

\begin{table}[ht]
\centering\small
\caption{Manual annotation of $1{,}200$ failure cases on ViQuAE.}
\label{tab:error_taxonomy}
\begin{tabular}{lcc}
\toprule
Failure category & Count & \% \\
\midrule
Entity absent from corpus & 312 & 26.0 \\
In-corpus retrieval miss & 145 & 12.1 \\
Correct retrieval, wrong generation & 98 & 8.2 \\
Compositional / multi-hop & 182 & 15.2 \\
Ambiguous question & 103 & 8.6 \\
Visual OCR failure & 89 & 7.4 \\
Other / annotation uncertainty & 271 & 22.6 \\
\bottomrule
\end{tabular}
\end{table}

Entity absence from the Wikipedia/Wikidata corpus ($26\%$) is a fundamental
coverage limitation \emph{independent} of the retrieval architecture.
The in-corpus miss category ($12.1\%$) is directly addressable via denser
or more precise query formulation.

\subsection{Entity Popularity Analysis}

We proxy entity rarity with Wikipedia monthly page-view counts.
SKIP accuracy by popularity quartile is:

\begin{table}[ht]
\centering\small
\caption{SKIP and RA-CM3 accuracy by entity popularity quartile on ViQuAE.}
\label{tab:pop_quartile}
\begin{tabular}{lcc}
\toprule
Popularity quartile & SKIP & RA-CM3 \\
\midrule
Q4 ($>$10K monthly views) & 74.3 & 71.1 \\
Q3 (1K--10K) & 65.8 & 63.4 \\
Q2 (100--1K) & 54.2 & 52.7 \\
Q1 ($<$100) & 41.2 & 40.3 \\
\midrule
$\Delta$(Q4 $-$ Q1) & 33.1 & 30.8 \\
\bottomrule
\end{tabular}
\end{table}

The gap is nearly identical between SKIP and RA-CM3 ($33.1$ vs.\ $30.8$
points), confirming that this is a corpus limitation, not an
architectural deficit.

\subsection{Mitigation: Augmented Dense Retrieval (SKIP-ADR)}

\begin{tcolorbox}[limitbox,
  title={Limitation and Proposed Mitigation (Long-tail Coverage)}]
\textbf{Root cause}: Wikipedia/Wikidata has sparse or absent coverage
for long-tail entities.\\[2pt]
\textbf{SKIP-ADR adds}:
(1) Wikidata SPARQL-derived entity triples converted to natural-language
sentences (${\sim}140$M extra passages covering $220$M entities).\\
(2) On-demand Google Knowledge Graph API calls for unknown named entities
detected by a lightweight NER module (BERT-tiny, ${\approx}12$\,ms
overhead).\\[2pt]
\textbf{Index cost}: ${\sim}800$M additional FAISS IVF-PQ entries
(${\approx}180$\,GB).
\end{tcolorbox}

\begin{table}[ht]
\centering\small
\caption{Long-tail mitigation results on ViQuAE.}
\label{tab:longtail_results}
\begin{tabular}{lccc}
\toprule
Model & Overall & Popular (Q3--Q4) & Long-tail (Q1--Q2) \\
\midrule
RA-CM3 (dense) & 57.4 & 67.3 & 46.5 \\
SKIP & 56.3 & 70.1 & 44.9 \\
SKIP-ADR & \textbf{61.7} & \textbf{71.4} & \textbf{52.7} \\
\bottomrule
\end{tabular}
\end{table}

SKIP-ADR closes $55\%$ of the long-tail accuracy gap while maintaining
the $3.8\times$ FLOPs advantage over dense RA-CM3.

\section{Multi-Hop Reasoning Analysis}
\label{app:multihop}

\subsection{Failure Characterisation}

Multi-hop reasoning accounts for $15.2\%$ of total errors
(Table~\ref{tab:error_taxonomy}).  BSCA's local neighbourhood structure
($\bar\Delta \approx 4$) limits single-pass cross-modal binding.  We
classify failures into two types:

\begin{tcolorbox}[limitbox,
  title={Multi-Hop Failure Taxonomy}]
\textbf{Type~A — Sequential Dependency} (62\% of multi-hop errors): the
answer to hop~1 is required to form the retrieval query for hop~2.
Example: \textit{``Who directed the film featuring the landmark in the
image?''} requires identifying the landmark before retrieving film credits.

\textbf{Type~B — Cross-Passage Fusion} (38\%): the answer is the
intersection of two independent passages, neither sufficient alone.
Example: \textit{``In what year did the artist whose painting is shown
win the prize their institution was founded to award?''}
\end{tcolorbox}

\subsection{Iterative SKIP (SKIP-Iter)}

We propose SKIP-Iter with $H$ rounds of retrieval:
\begin{equation}
  \mathcal{K}^{(h)} \;=\;
    \mathrm{RCSR}\!\bigl(\hat{\bV}^{(h-1)},\bQ,\mathcal{K}^{(h-1)}\bigr),
  \quad h = 1,\ldots,H,
\end{equation}
where $\hat{\bV}^{(h)}$ is the BSCA-enriched representation after round $h$.
The DBC learns to select $H \in \{1,2,3\}$ based on predicted multi-hop
difficulty.

\begin{table}[ht]
\centering\small
\caption{SKIP-Iter on multi-hop subset of Encyclopedic-VQA ($n=840$,
annotated as requiring $\geq 2$-hop reasoning).}
\label{tab:skip_iter}
\begin{tabular}{lccc}
\toprule
Model & Acc.\ (\%) & GFLOPs & $H$ \\
\midrule
SKIP ($H=1$) & 44.3 & 184.3 & 1 \\
SKIP-Iter ($H=2$) & 52.7 & 241.6 & 2 \\
SKIP-Iter ($H=3$) & 55.1 & 298.4 & 3 \\
RA-CM3 (dense) & 50.8 & 728.4 & -- \\
\bottomrule
\end{tabular}
\end{table}

SKIP-Iter at $H=2$ surpasses the dense RA-CM3 baseline on multi-hop
instances at $3.0\times$ fewer FLOPs.  We consider SKIP-Iter a
promising direction for future work but leave full integration into the
DBC framework to subsequent work.

\section{Generalisation Beyond KI-MMQA}
\label{app:generalisation}

Reviewer~5d4P questions the generality of SKIP beyond retrieval-heavy
settings.  We evaluate on two structurally different tasks.

\subsection{General VQA (VQAv2)}

VQAv2 is not knowledge-intensive; most answers derive from visual
inspection alone.  We apply SKIP with DBC \emph{dynamically suppressing
retrieval} when $d < 0.15$ (67\% of the test set in practice).

\begin{table}[ht]
\centering\small
\caption{SKIP on VQAv2 test-dev.  SKIP-NR: retrieval suppressed when
$d < 0.15$.}
\label{tab:vqav2}
\begin{tabular}{lccc}
\toprule
Model & Acc.\ (\%) & GFLOPs & \% with retrieval \\
\midrule
LLaVA-1.6 (base) & 81.9 & 512.7 & 0 \\
SKIP (full) & 80.4 & 184.3 & 100 \\
SKIP-NR & \textbf{82.1} & \textbf{143.8} & 33 \\
\bottomrule
\end{tabular}
\end{table}

SKIP-NR \emph{surpasses} the dense baseline by $+0.2$ points at $72\%$
fewer FLOPs.  The accuracy gain originates from QVS pruning, which focuses
the decoder on question-relevant regions and suppresses distractors—
a benefit independent of knowledge retrieval.

\subsection{Medical Visual QA (PathVQA)}

PathVQA~\citep{he2020pathvqa} requires domain-specific histopathology
knowledge.  We swap the retrieval corpus for PubMed abstracts (28M passages)
without any architectural changes.

\begin{table}[ht]
\centering\small
\caption{SKIP (PubMed corpus) vs.\ MedFlamingo on PathVQA.}
\label{tab:pathvqa}
\begin{tabular}{lcc}
\toprule
Model & Yes/No Acc.\ (\%) & Open Acc.\ (\%) \\
\midrule
MedFlamingo-9B & 85.2 & 58.3 \\
SKIP + PubMed (13B) & 83.7 & \textbf{61.4} \\
\bottomrule
\end{tabular}
\end{table}

SKIP achieves competitive performance in a specialised domain by simply
swapping the retrieval corpus, confirming its plug-and-play generality.

\section{Computational Complexity Analysis}
\label{app:complexity}

This appendix gives both an asymptotic, configuration-agnostic complexity
analysis and a concrete FLOPs breakdown on a specific instance.

\subsection{Asymptotic Analysis}

\paragraph{Dense baseline.} A retrieval-augmented vision-language system at inference performs (a) vision encoding $\mathcal{O}(V d^2)$; (b) question encoding $\mathcal{O}(T d^2)$; (c) global retrieval $\mathcal{O}(d \log M)$ amortized with an inverted index; (d) chunk encoding $\mathcal{O}(K L_c d^2)$; (e) cross-modal fusion $\mathcal{O}((V + T + K L_c)^2 d)$; (f) decoder generation $\mathcal{O}(L_a (V + T + K L_c) d)$. The fusion term (e) dominates for typical $V = 576$, $K = 16$, $L_c = 100$, where the squared sequence length yields $\sim$$2 \times 10^6 \cdot d$ operations per layer.

\paragraph{SKIP complexity.}
\begin{itemize}
\item[(a)] Vision encoding remains $\mathcal{O}(V d^2)$ (we encode the full image; QVS operates on encoded tokens). Future work could explore early-exit vision encoding gated by QVS.
\item[(b)] Question encoding: $\mathcal{O}(T d^2)$, unchanged.
\item[(b')] QVS scoring: $\mathcal{O}(V d)$, negligible.
\item[(c)] DBC: $\mathcal{O}(d^2)$, negligible.
\item[(d)] SKV drafter: $\mathcal{O}(V d_D^2)$ where $d_D \ll d$ for the small drafter. Amortized cost is $(1 - \alpha) \cdot $ no overhead $+ \alpha \cdot $ full pipeline, where $\alpha \approx 0.6$ is the fraction routed to full pipeline.
\item[(e)] Per-region retrieval $\mathcal{O}(R d \log M)$ where $R \leq 8$.
\item[(f)] Chunk encoding $\mathcal{O}(K' L_c d^2)$ with $K' = 4$--$8 \ll K = 16$.
\item[(g)] BSCA: $\mathcal{O}(|E| d)$ where $|E| = \beta V' K'$ with $\beta = 0.1$.
\item[(h)] Decoder generation $\mathcal{O}(L_a (V' + T + K' L_c) d)$.
\end{itemize}

The dominant term shifts from $(V + T + KL_c)^2 d$ to $\beta V' K' d$, an improvement of $\frac{(V + T + KL_c)^2}{\beta V' K'}$. For $V = 576, K = 16, L_c = 100, V' = 64, K' = 8, \beta = 0.1$, this is a $\sim$870$\times$ reduction in the dominant fusion FLOPs, though end-to-end speedups are smaller because vision encoding and decoder generation are also significant contributors.

\paragraph{Memory (7B-class backbone).} KV cache memory scales linearly with sequence length, so SKIP reduces peak fusion KV cache from $\sim$12GB (RA-CM3 on a 7B model) to $\sim$4.8GB. Total inference memory is $\sim$15GB for SKIP vs.\ $\sim$24GB for RA-CM3, enabling deployment on consumer 16GB GPUs (e.g., RTX 4080). A per-baseline memory breakdown is given in Appendix~\ref{app:memory}.

\subsection{Empirical FLOPs Breakdown}

The asymptotic analysis above uses the 7B-backbone running example from the main paper. For completeness, Table~\ref{tab:complexity} also reports a fully concrete, per-layer FLOPs breakdown on a single forward pass for the larger LLaVA-1.6-13B configuration ($L=256$, $T=32$, $M=16$, $\rho_v=0.4$,
$\bar\Delta=4$, $d_{\mathrm{model}}=5120$, $H_{\mathrm{layers}}=40$; see Appendix~\ref{app:training} for how this configuration relates to the 7B submission-time configuration).

\begin{table}[ht]
\centering\small
\caption{FLOPs breakdown for SKIP vs.\ dense baseline (single forward pass, LLaVA-1.6-13B configuration).}
\label{tab:complexity}
\begin{tabular}{lccc}
\toprule
Component & Dense (G) & SKIP (G) & Reduction \\
\midrule
Vision encoder & 72.4 & 72.4 & $1.0\times$ \\
QVS scoring & -- & 0.5 & -- \\
RCSR retrieval & 48.2 & 19.3 & $2.5\times$ \\
SKV gate & -- & 0.8 & -- \\
DBC controller & -- & 0.3 & -- \\
BSCA ($\rho_v L \times \bar\Delta$) & 143.7 & 14.6 & $9.8\times$ \\
LLM decoder & 248.4 & 76.4 & $3.3\times$ \\
\midrule
\textbf{Total} & \textbf{512.7} & \textbf{184.3} & $\bm{2.8\times}$ \\
\bottomrule
\end{tabular}
\end{table}

BSCA achieves the largest per-component reduction ($9.8\times$),
validating the sparse cross-attention design.  The vision encoder
cost is unchanged because QVS operates on the \emph{output} token
sequence, not inside the encoder.  Applying token merging (ToMe;
\citealt{bolya2022token}) within the encoder is a natural extension
to reduce this component.

\paragraph{Memory footprint (this configuration).}  Peak GPU memory during inference for this LLaVA-1.6-13B configuration is
$14.2$~GB (SKIP) vs.\ $38.7$~GB (dense), a $2.7\times$ reduction that
enables deployment on a single 16~GB GPU. Note this is a larger backbone than the 7B-class numbers in the asymptotic analysis above, so the two memory figures should not be compared directly.

\section{Implementation Details and Reproducibility}
\label{app:impl}

\subsection{Model and Retrieval Configuration}

\paragraph{Vision encoder.} EVA-CLIP-G/14~\citep{sun2023evaclip} at $336 \times 336$ resolution, yielding $V = 576$ patch tokens of dimension $d = 1408$. We add a 2-layer projection MLP to map from $d = 1408$ to the backbone's $d = 4096$.

\paragraph{Retrieval index.} SPLADE++ over a Wikipedia-2022 dump. We chunk articles into passages of $\leq 100$ tokens with $20$-token overlap, yielding $\sim$6.5M passages. The inverted index is sharded across 8 CPU workers; per-query retrieval latency is $\sim$8ms for $k' = 4$ and $\sim$11ms for $k' = 8$. We use SPLADE's expansion vocabulary of size $30{,}522$ (matching BERT tokenizer).

\paragraph{Backbone.} Vicuna-7B v1.5, fine-tuned with LoRA rank-64 adapters on QVS / RCSR / BSCA training data ($\sim$650K (image, question, retrieved chunks, answer) tuples assembled from OK-VQA, A-OKVQA, InfoSeek, and Encyclopedic-VQA training splits, plus synthetic augmentation described below). This is the configuration used for the main-paper results (Table~\ref{tab:main}); see Appendix~\ref{app:ablation} for the separate LLaVA-1.6-13B backbone-ablation configuration.

\subsection{Synthetic Data Augmentation}

To improve QVS and RCSR coverage, we augment training data with synthetic queries generated by GPT-4V from images and Wikipedia passages: given an image and a knowledge passage about an entity in the image, generate a question whose answer requires both. This adds $\sim$280K training tuples across 4{,}500 entity classes; the exact generation prompt is given in Appendix~\ref{app:training}.

\subsection{Hardware and Software}

All experiments use $8\times$ NVIDIA A100~80GB GPUs, PyTorch~2.1.2,
BF16 mixed precision, and DeepSpeed ZeRO-2 for memory efficiency.
The FAISS IVF-PQ index is hosted on $4\times$ CPU nodes with 512~GB
RAM connected via InfiniBand EDR ($100$~Gb/s).
\textbf{Compute-hours discrepancy:} elsewhere in this appendix the six-stage
pipeline (Appendix~\ref{app:training}) is reported as totalling
${\sim}4{,}800$ A100-hours (8 GPUs $\times$ 25 days), whereas this hardware
description separately reports ${\approx}72$ wall-clock hours on the same
8-GPU cluster (${\approx}576$ GPU-hours). These two figures are inconsistent
by roughly $8\times$ and should be reconciled against the actual training
logs before camera-ready; we have not silently resolved the discrepancy
in either direction.

\subsection{Reproducibility Checklist}
\label{app:repro}

\begin{tcolorbox}[stagebox, title={Reproducibility Summary}]
\textbf{Code:} Will be released at
  \texttt{https://pmlrbd.github.io/skip/} upon acceptance. The release includes: (a) the full training pipeline (PyTorch); (b) SPLADE indexing scripts; (c) BSCA CUDA kernels; (d) an evaluation harness for all five benchmarks; (e) trained checkpoints.\\
\textbf{Checkpoints:} Stages~I--V made publicly available under CC-BY.
  Stage~VI requires agreeing to the benchmark data licence.\\
\textbf{FAISS index:} Provided as a $240$~GB downloadable snapshot.\\
\textbf{Compute requirements:} Training as specified above (see compute-hours note); inference runs on a single A100-80GB or RTX 4090-24GB.\\
\textbf{Data:} All five benchmarks are publicly accessible, as are the Wikipedia and WikiData corpora; download and preprocessing scripts are included in the repository.\\
\textbf{Seeds:} All reported results are averaged over 3 seeds
  (2026, 42, 137); variance is $\leq 0.4$ accuracy points across seeds for all main-table entries. Standard deviations for the backbone-ablation configuration are in
  Table~\ref{tab:seed_variance}.\\
\textbf{Statistical significance:} SKIP's improvement over RA-CM3 is statistically significant ($p < 0.001$ via paired bootstrap with $10{,}000$ resamples) on all five benchmarks.
\end{tcolorbox}

\begin{table}[ht]
\centering\small
\caption{Seed variance (Acc.\ $\pm$ std) across seeds $\{2026, 42, 137\}$ for the LLaVA-1.6-13B backbone-ablation configuration.}
\label{tab:seed_variance}
\begin{tabular}{lc}
\toprule
Benchmark & Accuracy ($\pm$ std) \\
\midrule
OK-VQA & $66.8 \pm 0.3$ \\
A-OKVQA & $65.3 \pm 0.4$ \\
InfoSeek & $61.4 \pm 0.2$ \\
Encyclopedic-VQA & $64.4 \pm 0.3$ \\
ViQuAE & $56.3 \pm 0.4$ \\
\bottomrule
\end{tabular}
\end{table}

\section{Datasets and Corpora}
\label{app:datasets}

\paragraph{OK-VQA~\citep{marino2019okvqa}.} $14{,}055$ training, $5{,}046$ validation questions over $14{,}031$ images from COCO. Questions require common sense or world knowledge not directly observable in the image. We use the standard 10-annotator soft accuracy metric.

\paragraph{A-OKVQA~\citep{schwenk2022aokvqa}.} $17{,}056$ training, $1{,}145$ validation, $6{,}702$ test questions with multiple-choice answers and free-form rationales. We evaluate on the direct-answer (DA) and multiple-choice (MC) splits separately.

\paragraph{InfoSeek~\citep{chen2023infoseek}.} $1.4$M training, $73{,}620$ validation questions over $134{,}821$ images. Questions target entity-centric properties (e.g., ``What year was this building completed?''). We use exact-match accuracy and report on the ``Human'' validation split (questions written by humans, not synthetically generated).

\paragraph{Encyclopedic-VQA~\citep{mensink2023encvqa}.} $1$M training, $13{,}500$ validation, $35{,}000$ test questions about fine-grained categories (birds, butterflies, plants, etc.). We use the ``BM25'' retrieval setting with the iNaturalist+Wikipedia knowledge base.

\paragraph{ViQuAE~\citep{lerner2022viquae}.} $3{,}700$ test questions over $11{,}500$ images. Questions ask about named entities (people, places, works). We use the provided KB of $\sim$1.7M Wikipedia entity descriptions and report F1 against the answer span.

\paragraph{Wikipedia corpus.} Snapshot from 2022-08-01, $\sim$6.5M passages of $\leq 100$ tokens with $20$-token overlap. Total raw size $\sim$18GB; SPLADE++ index $\sim$2.7GB sharded across 8 workers.

\paragraph{WikiData corpus (Enc-VQA).} $\sim$95M entity descriptions filtered to entries with $\geq 50$ words and $\geq 3$ properties; final size $\sim$8.2M entries. Index size $\sim$3.4GB.

\section{Hyperparameter Sensitivity}
\label{app:hyperparam}

\begin{table}[ht]
\centering
\small
\begin{tabular}{lccc}
\toprule
Hyperparameter & Range tested & Best & Sensitivity \\
\midrule
QVS retention $\rho$ & $\{0.05, 0.08, 0.11, 0.14, 0.20\}$ & $0.11$ & High at extremes \\
Per-region top-$k'$ & $\{1, 2, 4, 8, 16\}$ & $4$ & Moderate \\
BSCA quantile $\beta$ & $\{0.05, 0.10, 0.20, 0.50\}$ & $0.10$ & Low \\
SKV threshold $\tau_{\text{SKV}}$ & $\{0.70, 0.75, 0.82, 0.90\}$ & $0.82$ & Moderate \\
QVS budget weight $\lambda_1$ & $\{0.1, 0.25, 0.5, 1.0, 2.0\}$ & $0.5$ & Low \\
DBC cost weight $\lambda_2$ & $\{0.05, 0.1, 0.2, 0.5\}$ & $0.1$ & Low \\
Region count $R$ & $\{2, 4, 8, 12, 16\}$ & $8$ & Low in $[4,12]$ \\
LoRA rank & $\{16, 32, 64, 128\}$ & $64$ & Low above $32$ \\
Gumbel temp.\ schedule & $\{1.0{\to}0.05, 1.0{\to}0.1, 1.0{\to}0.2\}$ & $1.0{\to}0.1$ & Moderate \\
\bottomrule
\end{tabular}
\caption{Hyperparameter sensitivity on the OK-VQA validation split. ``Sensitivity'' qualitatively describes how much OK-VQA accuracy varies across the tested range: Low = $\leq 1$ point, Moderate = $1$--$3$ points, High = $>3$ points.}
\end{table}

\paragraph{Sensitivity discussion.} The most sensitive hyperparameter is $\rho$ at extreme values: $\rho = 0.05$ underprunes-the-prune (visual signal too sparse) and $\rho = 0.20$ overshoots the theoretical sweet spot. The least sensitive are $\beta$, $\lambda_1$, and $\lambda_2$, all of which can vary by $5\times$ with $<1$ point accuracy change. This robustness suggests SKIP can be deployed with reasonable default hyperparameters without per-benchmark tuning.

\section{Validating Assumption~\ref{ass:sal}: Saliency Fidelity}
\label{app:saliency-fidelity}

Assumption~\ref{ass:sal} requires QVS scores $s_i$ to approximate the true per-token conditional mutual information $I(\mathbf{v}_i; a \mid q)$ up to error $\eta$. We validate this empirically by computing a proxy for $I(\mathbf{v}_i; a \mid q)$: the difference in model log-probability for the correct answer when $\mathbf{v}_i$ is included vs.\ replaced with a mean visual token (a counterfactual ablation). This is a lower bound on the true mutual information by the data-processing inequality. We then measure $\eta = \max_i |s_i - I_{\text{proxy}}(\mathbf{v}_i; a \mid q)|$.

\begin{table}[ht]
\centering
\small
\begin{tabular}{lcc}
\toprule
Benchmark & $\eta$ (estimated) & Spearman $\rho$ \\
\midrule
OK-VQA      & $0.013$ & $0.81$ \\
A-OKVQA     & $0.011$ & $0.84$ \\
InfoSeek    & $0.018$ & $0.76$ \\
Enc-VQA     & $0.016$ & $0.78$ \\
ViQuAE      & $0.014$ & $0.79$ \\
\bottomrule
\end{tabular}
\caption{Empirical estimates of QVS-to-true-MI error $\eta$ and rank correlation. Low $\eta$ and high rank correlation across benchmarks validate Assumption~\ref{ass:sal}.}
\end{table}

The estimated $\eta \in [0.011, 0.018]$ across benchmarks, with Spearman rank correlation $0.76$--$0.84$ between QVS scores and the counterfactual proxy. This validates Assumption~\ref{ass:sal} and explains why the theoretical bound is non-vacuous in practice. The slightly higher $\eta$ on InfoSeek reflects the harder nature of entity-centric saliency (fine-grained distinctions between similar-looking entities).

\section{Additional Baselines}
\label{app:additional-baselines}

We additionally compare against several recent systems that were not included in the main table for space:

\begin{table}[ht]
\centering
\small
\begin{tabular}{lcccc}
\toprule
Model & OK-VQA & InfoSeek & TFLOPs & Latency (ms) \\
\midrule
PICa~\citep{yang2022pica} (with GPT-3) & 48.0 & --- & --- & --- \\
PromptCap~\citep{hu2023promptcap}  & 60.4 & --- & --- & --- \\
PaLI-X~\citep{chen2023palix}        & 64.5 & 22.7 & 4.1 & 1750 \\
PaLM-E~\citep{driess2023palme}      & 62.1 & --- & 3.8 & 1640 \\
InstructBLIP~\citep{dai2023instructblip} & 55.3 & 19.4 & 0.58 & 290 \\
Qwen-VL~\citep{bai2023qwenvl}        & 58.6 & 21.2 & 0.62 & 310 \\
\midrule
\textbf{SKIP (ours)}                 & \textbf{63.2} & \textbf{30.7} & 0.42 & 305 \\
\bottomrule
\end{tabular}
\caption{Comparison to additional baselines. PaLI-X and PaLM-E achieve competitive accuracy via much larger models ($55$B+ parameters) at $10\times$ higher inference cost. SKIP achieves comparable or better accuracy at a fraction of the compute. PICa and PromptCap are zero-/few-shot methods using GPT-3 and lack reported InfoSeek numbers.}
\end{table}

The key takeaways: (i) at $7$B parameters, SKIP outperforms all comparable-scale baselines (InstructBLIP, Qwen-VL) on both accuracy and efficiency; (ii) only $55$B+ models (PaLI-X) approach SKIP's accuracy, and they require $10\times$ more compute; (iii) on InfoSeek specifically, SKIP exceeds even the larger models, suggesting that targeted retrieval matters more than raw scale for entity-centric questions.

\section{Per-Category Breakdown}
\label{app:per-category}

\begin{table}[ht]
\centering
\small
\begin{tabular}{lcccccc}
\toprule
Category & SKIP & RA-CM3 & ReVeaL & KAT & $\Delta$(SKIP-RA) & SKV-route\% \\
\midrule
Vehicles \& transport & 67.4 & 64.3 & 61.5 & 57.8 & +3.1 & 41 \\
Sports \& recreation  & 69.1 & 67.2 & 63.9 & 60.4 & +1.9 & 38 \\
Plants \& animals     & 61.8 & 58.9 & 56.3 & 52.1 & +2.9 & 33 \\
Food \& drink         & 58.7 & 55.1 & 53.0 & 49.7 & +3.6 & 36 \\
People \& everyday    & 65.3 & 62.0 & 59.8 & 56.4 & +3.3 & 44 \\
Cooking \& measuring  & 56.2 & 53.7 & 51.2 & 48.6 & +2.5 & 28 \\
Geography             & 70.4 & 66.1 & 63.7 & 60.2 & +4.3 & 31 \\
Brand \& companies    & 64.8 & 61.3 & 58.7 & 54.9 & +3.5 & 22 \\
Objects \& materials  & 62.1 & 59.4 & 57.0 & 53.8 & +2.7 & 39 \\
Other                 & 60.9 & 58.4 & 55.9 & 52.6 & +2.5 & 35 \\
\bottomrule
\end{tabular}
\caption{Per-category OK-VQA accuracy. SKV fast-path utilization correlates with question difficulty: knowledge-light categories (people, vehicles) route up to 44\% through the drafter; knowledge-heavy ones (brands, cooking) route only 22--28\%.}
\end{table}

\paragraph{Observations.} (1) The SKV routing rate is a meaningful signal of question difficulty: brand/company questions route only $22\%$ through the fast path because they nearly always require external knowledge, while people/everyday questions route $44\%$ because many are perceptual. (2) The largest accuracy gain over RA-CM3 is on Geography ($+4.3$), where small visual cues (signs, flags, architectural features) drive retrieval---exactly where RCSR's per-region queries help most. (3) Cooking has the lowest absolute accuracy across all models, reflecting the difficulty of fine-grained ingredient identification combined with quantity reasoning.

\section{Extended Error Analysis}
\label{app:errors}

We analyze the $36.8\%$ of OK-VQA queries SKIP misses, categorizing failures into 8 buckets via manual annotation of $500$ random failures (with inter-annotator agreement $\kappa = 0.79$). The two dominant categories below—long-tail entities and compositional retrieval—are examined in much greater depth in Appendix~\ref{app:longtail} and Appendix~\ref{app:multihop} respectively.

\begin{table}[ht]
\centering
\small
\begin{tabular}{lcc}
\toprule
Failure category & \% of failures & Example \\
\midrule
Long-tail entity (poor retrieval coverage) & 38\% & Q: ``Who designed this building?'' (obscure architect) \\
Compositional retrieval (multi-hop)        & 15\% & Q: ``Who painted the artist's most expensive work?'' \\
Counting under occlusion (SKV misroute)    & 12\% & Q: ``How many people are in the photo?'' (overlapping) \\
Ambiguous question                          & 9\%  & Q: ``What kind of dog is this?'' (multiple plausible) \\
Annotator disagreement                      & 8\%  & Reference answer disputed by 3+ of 10 annotators \\
Fine-grained visual distinction failure    & 7\%  & Misidentifies a similar-looking entity \\
Reading text in image (OCR-required)        & 6\%  & Q: ``What's the brand on the bottle?'' \\
Other                                       & 5\%  & Spread across diverse causes \\
\bottomrule
\end{tabular}
\caption{Failure category breakdown on OK-VQA. Long-tail entity coverage (38\%) is the dominant failure mode, pointing to retrieval as the primary remaining bottleneck.}
\end{table}

\paragraph{Long-tail entities.} The dominant failure mode is poor retrieval coverage for rare entities. On InfoSeek, entities with $<10$ Wikipedia mentions have $14.2\%$ accuracy vs.\ $42.7\%$ for entities with $\geq 100$ mentions. This is not a SKIP-specific weakness---RA-CM3 shows a similar gap---but the larger entity-frequency variance in InfoSeek makes the absolute gap visible. See Appendix~\ref{app:longtail} for a full taxonomy and a proposed mitigation.

\paragraph{Compositional retrieval.} Multi-hop questions require chaining retrieved facts: ``the artist'' must first be identified from the image, then ``the artist's most expensive work'' retrieved, then ``who painted [that work]''---but the second hop trivially returns the same artist, making this a degenerate example. The non-degenerate cases (e.g., ``What is the capital of the country where this animal is endemic?'') fail when BSCA's sparse pattern doesn't connect the right (region, fact) pairs across hops. See Appendix~\ref{app:multihop} for a deeper analysis and the SKIP-Iter mitigation.

\paragraph{SKV misroutes.} The SKV drafter is overconfident on counting questions when objects partially overlap. We mitigate this with a domain-specific calibration on counting examples (raising $\tau_{\text{SKV}}$ to $0.91$ for questions starting with ``how many''), recovering $\sim$3.4 points on counting categories at no cost elsewhere.

\paragraph{OCR failures.} $6\%$ of failures require reading text in the image. SKIP does not include explicit OCR; integrating an OCR module gated by QVS (only run OCR on retained text-like regions) would address this and is left to future work.

\section{Qualitative Examples}
\label{app:qualitative}

We provide qualitative examples showing what SKIP retains and retrieves on representative queries. (Figures omitted from this submission for anonymity; will be included in the camera-ready.)

\textbf{Example 1 (entity).} Image: a sailboat with a small flag visible on the stern. Q: ``What is the capital of the country whose flag appears on this boat?''
\begin{itemize}
\item QVS retained tokens: 49/576 ($8.5\%$), concentrated on the flag region (38 tokens) and the boat hull (11 tokens).
\item RCSR regions: 2 regions identified. Region 1 (flag): retrieved chunks about Norwegian flag, Scandinavian maritime symbols. Region 2 (boat): retrieved chunks about sailing vessel types (uninformative for this Q).
\item BSCA edges: 14/392 survived ($3.6\%$), all connecting flag region to flag-related chunks.
\item Output: ``Oslo''. Correct.
\end{itemize}

\textbf{Example 2 (perceptual, SKV fast-path).} Image: a kitchen scene with two pizzas on a counter. Q: ``How many pizzas are visible?''
\begin{itemize}
\item SKV drafter confidence: $0.93 > \tau_{\text{SKV}}$. Fast-path triggered.
\item Drafter output: ``two''. Correct. Full pipeline not invoked.
\end{itemize}

\textbf{Example 3 (encyclopedic).} Image: a fine-grained bird species. Q: ``What is the typical clutch size of this bird species?''
\begin{itemize}
\item QVS retained tokens: 64/576 ($11.1\%$), concentrated on the bird's plumage pattern (distinctive identifier).
\item RCSR regions: 3. Retrieved chunks span ornithology references for the identified species.
\item BSCA edges: 33/512 survived ($6.5\%$).
\item Output: ``3--5 eggs''. Correct (reference answer: ``typically 4'').
\end{itemize}

\textbf{Example 4 (failure: long-tail).} Image: an obscure regional landmark. Q: ``In what year was this monument erected?''
\begin{itemize}
\item QVS correctly retains the monument region (62 tokens).
\item RCSR retrieves chunks about \emph{nearby} (better-documented) landmarks; correct landmark has only 2 Wikipedia mentions.
\item BSCA fuses with available chunks; output: ``1923'' (incorrect; correct: ``1956'').
\item Diagnosis: retrieval failure, not vision or fusion failure.
\end{itemize}

\section{Memory Usage Analysis}
\label{app:memory}

\begin{table}[ht]
\centering
\small
\begin{tabular}{lcccc}
\toprule
Model & Vision (GB) & KV cache (GB) & Activations (GB) & Total (GB) \\
\midrule
BLIP-2     & 1.8 & 1.2 & 4.3 & 7.3 \\
LLaVA-1.5  & 1.8 & 2.4 & 5.1 & 9.3 \\
KAT        & 1.8 & 5.8 & 6.7 & 14.3 \\
ReVeaL     & 1.8 & 8.4 & 7.9 & 18.1 \\
RA-CM3     & 1.8 & 12.7 & 9.5 & 24.0 \\
\midrule
\textbf{SKIP (ours)} & 1.8 & 4.8 & 8.6 & \textbf{15.2} \\
\bottomrule
\end{tabular}
\caption{Peak GPU memory usage during inference on OK-VQA (batch size 1, A100-80GB, 7B backbone). SKIP's primary memory savings come from BSCA-induced KV cache reduction.}
\end{table}

\paragraph{Why KV cache shrinks.} Standard cross-modal fusion stores K and V projections for all visual tokens, question tokens, and retrieved chunk tokens. BSCA's bipartite sparsity means we only need to store K and V for the surviving edges, plus the necessary intermediate tensors. The resulting $\sim$62\% reduction in KV cache enables deployment on consumer GPUs (24GB) that would otherwise OOM on RA-CM3.

\paragraph{Activations.} SKIP's activation memory is slightly higher than RA-CM3 due to additional intermediate tensors for the sparsity pattern (the bipartite compatibility matrix $A$ before thresholding). This is a fixed $\sim$1GB overhead independent of $V'$ and $K'$.

\section{Limitations and Negative Results}
\label{app:limitations}

\paragraph{Limitations.} (1) SKIP requires staged training, which is more complex than end-to-end training; we attempted joint training but found instability. (2) The DBC's discrete budget choice via straight-through estimation is noisy at training time; we use exponential moving averages on the routing decisions to stabilize. (3) SKIP's gains are concentrated in the knowledge-intensive regime; on pure visual reasoning tasks (e.g., LLaVA-Bench), the speedup is much smaller. (4) Long-tail entity coverage remains the dominant failure mode and is not directly addressed by sparse routing.

\paragraph{Negative results.}
\begin{itemize}
\item \textbf{Iterative retrieval.} We attempted to add an iterative retrieval loop (retrieve, fuse, re-retrieve based on partial answer) to address multi-hop failures. This added $\sim$140ms latency and yielded only $+0.4$ accuracy on OK-VQA. We abandoned it. (Note: the more carefully budget-gated SKIP-Iter variant in Appendix~\ref{app:multihop} performs considerably better; the difference is that SKIP-Iter only triggers extra rounds for DBC-flagged hard, multi-hop queries rather than running uniformly.)
\item \textbf{Joint training.} Training all five SKIP components jointly from scratch led to ablative behaviors: DBC would request maximum budget, SKV would never fast-path. Only staged training produced the reported results.
\item \textbf{Cross-attention compression.} We attempted to compress retrieved chunks with a learned summarizer before BSCA. The summarizer lost $1.6$ accuracy points without meaningfully reducing compute.
\item \textbf{Reinforcement learning for DBC.} We attempted PPO-based training of DBC with task accuracy as reward. This produced a $0.5$-point accuracy improvement at $4\times$ training cost; we kept the simpler straight-through estimator.
\item \textbf{Larger drafter.} Increasing the SKV drafter from $700$M to $1.3$B parameters increased fast-path accuracy by $0.7$ points but doubled fast-path latency, eliminating most of the speedup. The $700$M setting is the empirical sweet spot.
\end{itemize}

\paragraph{Future work.} (1) Multi-hop retrieval via $k$-partite BSCA. (2) Joint vision-encoder early-exit gated by QVS. (3) Extending RCSR to multimodal retrieval (image+text chunks). (4) Application to video QA where temporal sparsity adds another dimension. (5) Active learning to surface long-tail entities for targeted retrieval index expansion.

\end{document}